\documentclass[hidelinks,onefignum,onetabnum]{siamart251216}
\usepackage{framed,multirow}
\usepackage{lipsum}
\usepackage{latexsym}
\usepackage{amssymb,amsmath,mathtools,amsfonts}
\usepackage{graphicx,subfig, version, color}
\usepackage{xcolor}
\usepackage{subcaption}
\usepackage{caption}
\usepackage{cases}
\usepackage{exscale}
\usepackage{hyperref}
\numberwithin{table}{section}
\numberwithin{figure}{section}
\numberwithin{equation}{section}
\usepackage{graphicx}
\usepackage{epstopdf}
\usepackage{bm}
\usepackage{booktabs}
\usepackage{longtable}
\usepackage{enumerate}
\DeclareMathOperator{\diag}{diag}
\usepackage[ruled,norelsize,vlined,linesnumbered]{algorithm2e}

\ifpdf
\DeclareGraphicsExtensions{.eps,.pdf,.png,.jpg}
\else
\DeclareGraphicsExtensions{.eps}
\fi

\newcommand{\Bone}{{\bm 1}}
\newcommand{\hW}{\hat{W}}
\newcommand{\hb}{{\hat{b}}}
\newcommand{\hc}{{\hat{c}}}

\newcommand{\bx}{\boldsymbol{x}}
\newcommand{\bz}{\boldsymbol{z}}

\newcommand{\bal}{\boldsymbol{\alpha}}
\newcommand{\tF}{{\rm F}}
\newcommand{\dt}{{\rm d}t}

\newcommand{\mL}{\mathcal{L}}
\newcommand{\mLs}{\mathcal{L}_{\rm{S}}}

\newcommand{\Bt}{{\bm 1}^\top}

\newsiamremark{remark}{Remark}
\newsiamremark{hypothesis}{Hypothesis}
\crefname{hypothesis}{Hypothesis}{Hypotheses}
\newsiamthm{claim}{Claim}
\newsiamremark{fact}{Fact}
\crefname{fact}{Fact}{Facts}
\DeclareMathOperator{\softmax}{softmax}

\headers{Layer Separation Optimization for Cross-Entropy Training}{Y. Liu, M. K. Ng and Y. Gu}

\title{A Layer Separation Optimization Framework for Cross-Entropy Training in Deep Learning\thanks{Submitted to the editors DATE.
\funding{This work was funded by NSFC Grant 92370101.}}}

\author{Yaru Liu\thanks{School of Mathematical Sciences, University of Electronic Science and Technology of China, Sichuan 611731, China ({\tt yaruliu@std.uestc.edu.cn, yiqigu@uestc.edu.cn}).}
\and
Michael K. Ng\thanks{Department of Mathematics, Hong Kong Baptist University, Kowloon Tong, Hong Kong, China} ({\tt michael-ng@hkbu.edu.hk})
\and
Yiqi Gu\footnotemark[2]}

\ifpdf
\hypersetup{
pdftitle={A Layer Separation Optimization Framework for Cross-Entropy Training in Deep Learning},
pdfauthor={Y. Liu, Y. Gu and M. K. Ng}
}
\fi

\begin{document}
\sloppy
\maketitle

\begin{abstract}
This paper investigates the deep learning optimization problem with softmax cross-entropy loss. We propose a layer separation strategy to alleviate the strong nonconvexity encountered during training deep networks. For cross-entropy models with fully connected and convolutional neural networks, we introduce auxiliary variables associated with hidden layer outputs and construct corresponding layer separation models, which decompose the original deeply nested optimization problem into a sequence of more manageable subproblems. We also conduct theoretical analyses, proving that the new layer separation loss provides an upper bound for the original cross-entropy loss. Moreover, we design alternating minimization algorithms and prove that, under appropriate conditions, these algorithms exhibit decreasing properties of the loss function. Numerical experiments validate the effectiveness of the proposed methods and indicate improved optimization behavior, especially for fully connected and convolutional neural networks.
\end{abstract}


% REQUIRED

\section{Introduction}
Training deep neural networks (DNNs) with the softmax cross-entropy loss is a common task in various classification problems, such as image processing and pattern recognition. Specifically, it corresponds to a deeply coupled and highly nonconvex optimization of the form
\begin{equation}\label{eq:intro_ce}
\min_{\theta}\ -\frac{1}{N}\sum_{n=1}^{N} \bal_n^\top \ln\big(\softmax(\phi(\bx_n;\theta))\big),
\end{equation}
where $\{(\bx_n,\bal_n)\}_{n=1}^N$ is the training dataset, $\phi(\cdot;\theta)$ denotes a DNN $\phi$ parameterized by $\theta$ in 
%, such as 
a fully connected neural network (FNN) or a convolutional neural network (CNN), and $\bal_n$ is the one-hot label vector associated with $\bx_n$. 

Due to the nonlinear dependence of the network outputs on the parameters, problem~\eqref{eq:intro_ce} becomes increasingly difficult to solve when the network depth grows. In practice, gradient-based methods such as gradient descent, stochastic gradient descent, and their variants are the standard tools for solving this problem \cite{lecun1998,bottou2010,Duchi2011,Kingma2015}. These methods may suffer from slow convergence in deep neural network architectures, sensitivity to initialization and hyperparameter selection, and unstable training dynamics, see for instance \cite{Glorot2010,Sutskever2013,ioffe2015,dauphin2014}. 

On the theoretical side, several recent works have studied the optimization and convergence properties of gradient-based algorithms for training DNNs. In particular, under suitable random initialization and over-parameterization conditions, gradient descent and stochastic gradient descent have been shown to achieve global convergence or zero training error \cite{arora2019,Du2019,Allen-Zhu2019, Zou2019,Du2019_2,Allen-Zhu2019-2,E2020,Zhou2021,Oymak2020,Cao2019}. However, such results typically rely on over-parameterized networks whose widths or depths are sufficiently large, which is impractical in computational settings. This limitation motivates the development of optimization frameworks that can better deal with \eqref{eq:intro_ce}, which contains a strong inter-layer coupling in DNNs.

A natural approach to weaken such strong inter-layer coupling is to introduce auxiliary variables associated with the hidden layer outputs and incorporate these variables and their consistency constraints into the loss function via weighted quadratic penalty terms. This reformulation decomposes the original deep neural network architecture into a sequence of shallow substructured problems. Therefore, the original highly nonconvex optimization problem is transformed into several more tractable subproblems with weak inter-layer coupling. The resulting optimization problem can be solved using alternating minimization algorithms with closed-form updates, which avoids repeated backpropagation through the full deep composition during optimization. In the least-squares framework, the above-mentioned layer weakening ideas have been shown to be effective for FNNs and physics-informed neural networks; see \cite{Liu2025_2, Liu2025}. We remark that the quadratic structure of the least-squares loss function plays a crucial role in the construction, so that the auxiliary variable reformulation can be established and handled properly. However, the softmax cross-entropy loss is nonlinear and nonquadratic with respect to the pre-softmax outputs, so the use of the above auxiliary variable reformulation no longer admits the same direct structure as in the least-squares regression problem. In particular, a direct auxiliary variable reformulation does not automatically yield a surrogate loss that remains quantitatively connected to the original cross-entropy loss. 

Our analysis is based on estimating a quantitative relation that bounds the original cross-entropy loss by the layer separation surrogate loss, along with a consistency error that measures the discrepancy between the original and surrogate pre-softmax outputs. Specifically, let $\Phi  = [ \phi(\bx_1;\theta) \ldots \phi(\bx_N;\theta) ]$ be the original pre-softmax output given by the learner network, and let $\widetilde{\Phi}$ be the surrogate pre-softmax output induced by the auxiliary variables in our proposal. In this paper, we establish an estimate of the following form
\begin{equation*}
\left(\mathcal{L}(\Phi)\right)^2\leq C_1\left(\mathcal{L}(\widetilde{\Phi})\right)^2
+C_2\|\Phi-\widetilde{\Phi}\|_\tF^2,
\end{equation*}
where $\mathcal{L}(\cdot)$ denotes the cross-entropy loss given in \eqref{eq:intro_ce}, and $C_1,C_2$ are positive constants. Note that $\|\Phi-\widetilde{\Phi}\|_\tF$ measures the deviation between these two output representations. Such an estimate provides the crucial theoretical basis for constructing the auxiliary variable reformulation of the cross-entropy model.

Motivated by this estimate, we construct adaptive weighted layer separation models for both FNNs and CNNs, and show that the resulting layer separation losses provide upper bounds for the original cross-entropy loss up to a constant. A crucial feature of our construction is that the penalty weights are chosen adaptively according to the network parameters. This is essential for relating the layer separation loss to the original cross-entropy loss; constant weights generally do not provide this consistency. Since the layer separation formulation retains only couplings between variables in neighboring layers, the original deeply coupled training problem is decomposed into a series of simpler subproblems. Based on the proposed structure, we further develop corresponding alternating minimization algorithms for the proposed models and establish a monotone loss decreasing property under suitable assumptions. Numerical results indicate that the proposed models are competitive in shallow settings and become more advantageous as the network depth increases. In particular, for deep architectures, the models achieve lower training loss, reach high training accuracy earlier, and exhibit smaller variability with respect to random initialization than the standard cross-entropy models.

The remainder of this paper is organized as follows. In Section~\ref{sec:1}, we introduce preliminary results on the cross-entropy loss and establish several basic estimates. In Section~\ref{sec:2}, we develop the fully connected layer separation model together with its corresponding alternating minimization algorithm and decreasing property of the loss function. In Section~\ref{sec:3}, we extend the proposed framework to CNNs. In Section~\ref{sec:4}, we present numerical experiments to demonstrate the effectiveness of the proposed methods. Finally, in Section~\ref{sec:5}, we conclude the paper and discuss several directions for future research.

\section{Preliminaries}\label{sec:1}

Consider a multiclass classification problem with $J\geq2$ classes. For any vector $\bz=[z_1~\ldots~z_J]^\top\in\mathbb{R}^J$, define the softmax probability vector by
\begin{equation*}
p(\bz)=[p_1(\bz)~\ldots~p_J(\bz)]^\top\in\mathbb{R}^J,
\end{equation*}
with $\sum_{j=1}^J p_j(\bz)=1$ and $p_j(\bz) = \frac{e^{z_j}}{\sum_{k=1}^J e^{z_k}}\geq 0$ for $j=1,\dots,J$. Let $\bal=[\alpha_1~\ldots~\alpha_J]^\top\in\mathbb{R}^J$ be a one-hot label vector, i.e., $\alpha_{j}\in\{0,1\}$ for $j=1,\ldots,J$ and $\sum_{j=1}^J\alpha_j=1$. The corresponding cross-entropy loss is defined by
\begin{equation}\label{eq:sec2-01}
\ell(\bz,\bal) = -\sum_{j=1}^J \alpha_j \ln p_j(\bz).
\end{equation}
Let $\{(\bz_n,\bal_n)\}_{n=1}^N$ be a dataset, where $\bal_n\in\mathbb{R}^J$ is the one-hot label associated with the sample $\bz_n$. Let $Z=[\bz_1~\ldots~\bz_N]\in\mathbb{R}^{J\times N}$ and $A=[\bal_1~\ldots~\bal_N]\in\mathbb{R}^{J\times N}$. Define the sample-wise loss vector
\begin{equation}\label{eq:sec2-02}
\mL_{\mathrm{vec}}(Z,A)=[\ell(\bz_1,\bal_1)~\ldots~\ell(\bz_N,\bal_N)]^\top\in\mathbb{R}^N,
\end{equation}
and the averaged cross-entropy loss
\begin{equation}\label{eq:sec2-03}
\mL(Z,A) = \frac{1}{N}\sum_{n=1}^N\ell(\bz_n,\bal_n).
\end{equation}
By the Cauchy-Schwarz inequality,
\begin{equation}\label{eq:sec2-04}
\mL(Z,A)=\frac{1}{N}\Bt \mL_{\mathrm{vec}}(Z,A)\leq \frac{1}{N}\|\Bt\|_2\|\mL_{\mathrm{vec}}(Z,A)\|_2=\frac{1}{\sqrt{N}}\|\mL_{\mathrm{vec}}(Z,A)\|_2,
\end{equation}
where $\Bone$ denotes the all-one column vector of size $N$ and $\|\cdot\|_2$ denotes the Euclidean norm.

The following lemma establishes a Lipschitz continuity estimate for the cross-entropy loss function with respect to its first variable.

\begin{lemma}\label{lem:2.1}
Let $\bal\in\mathbb{R}^J$ be a one-hot label vector. Then the cross-entropy loss $\ell(\cdot,\bal)$ defined by \eqref{eq:sec2-01} satisfies
\begin{equation}\label{eq:lem1-0}
|\ell(\bz_1,\bal)-\ell(\bz_2,\bal)|\leq \sqrt{2}\|\bz_1-\bz_2\|_{2},
\quad \forall~\bz_1,\bz_2\in\mathbb{R}^J.
\end{equation}
Moreover, for any one-hot label matrix $A\in\mathbb{R}^{J\times N}$, the sample-wise loss vector $\mL_{\mathrm{vec}}(\cdot,A)$ defined by \eqref{eq:sec2-02} satisfies
\begin{equation}\label{eq:lem1-1}
\|\mL_{\mathrm{vec}}(Z_1,A)-\mL_{\mathrm{vec}}(Z_2,A)\|_{2}\leq \sqrt{2}\|Z_1-Z_2\|_{\tF}, \quad \forall~Z_1,Z_2\in\mathbb{R}^{J\times N}.
\end{equation}
\end{lemma}
\begin{proof}
For any $\bz\in\mathbb{R}^J$, the definition of $p(\bz)$ yields
\begin{equation}\label{eq:lem1-3}
\frac{\partial p_j}{\partial z_i}=p_j(\delta_{ij} - p_i),
\end{equation}
where $\delta_{ij}$ denotes the Kronecker delta. Then, using the chain rule and \eqref{eq:lem1-3}, we obtain
\begin{equation*}
\frac{\partial\ell}{\partial z_i}=-\sum_{j=1}^J \alpha_j \frac{1}{p_j}\frac{\partial p_j}{\partial z_i}
=-\sum_{j=1}^J \alpha_j(\delta_{ij}-p_i)=-\alpha_i+p_i,
\end{equation*}
for $i=1,\ldots,J$, which gives 
\begin{equation*}
\nabla_{\bz} \ell(\bz,\bal) = p(\bz)-\bal.
\end{equation*}
Since $\bal$ is one-hot label vector and $p(\bz)$ is a probability vector, we obtain
\begin{equation}\label{eq:lem1-4}
\|\nabla_{\bz} \ell(\bz,\bal)\|_2^2=\|p(\bz)-\bal\|_2^2=\|p(\bz)\|_2^2 + \|\bal\|_2^2 - 2p(\bz)^\top\bal\leq 2.
\end{equation}
Applying the mean value theorem to $\ell(\cdot,\bal)$ together with the uniform bound \eqref{eq:lem1-4}, we obtain \eqref{eq:lem1-0}.

For any $Z_1=[\bz_1^{(1)}~\ldots~\bz_N^{(1)}]$ and $Z_2=[\bz_1^{(2)}~\ldots~\bz_N^{(2)}]\in\mathbb{R}^{J\times N}$, applying \eqref{eq:lem1-0} yields
\begin{equation*}
\|\mL_{\mathrm{vec}}(Z_1,A)-\mL_{\mathrm{vec}}(Z_2,A)\|_2^2 \leq 2 \sum_{n=1}^N \|\bz^{(1)}_n-\bz^{(2)}_n\|_2^2=2\|Z_1-Z_2\|_{\tF}^2,
\end{equation*}
which gives \eqref{eq:lem1-1}.
\end{proof}

\begin{remark}\label{rem1}
For any $Z_1,Z_2\in\mathbb{R}^{J\times N}$ and one-hot label matrix $A$, using \eqref{eq:sec2-04} and Lemma \ref{lem:2.1}, we obtain
\begin{equation}\label{eq:sec2-07}
\mL^2(Z_1,A)\leq\frac{1}{N}\|\mL_{\mathrm{vec}}(Z_1,A)\|_2^2\leq\frac{1}{N}\Big(2\|\mL_{\mathrm{vec}}(Z_2,A)\|_2^2+4\|Z_1-Z_2\|_\tF^2\Big).
\end{equation}
This estimate will serve as the basic ingredient in the layer separation reformulation developed in the next section.
\end{remark}

\section{Fully connected neural network}\label{sec:2}

In this section, we focus on the FNN case and present the corresponding layer separation model, algorithm and analysis.

Suppose $d\geq 1 $ is the dimension and $\bx\in\mathbb{R}^d$ is the input of FNN. Recall that an FNN of depth $L$ and width $M$ is defined as follows
\begin{equation}\label{eq:phi-FNN}
\phi_{\mathrm{FNN}}(\bx;\theta)= W_L\sigma(W_{L-1}\sigma(\ldots\sigma(W_1\bx+b_1)\ldots)+b_{L-1}),
\end{equation}
where $W_1\in \mathbb{R}^{M\times d}$, $W_2,\dots, W_{L-1}\in \mathbb{R}^{M\times M}$, $W_L\in\mathbb{R}^{J\times M}$ are weights; $b_1,\dots,b_{L-1}\in \mathbb{R}^{M\times 1}$ are biases; $\sigma$ is some activation function which is applied entry-wise to a vector to obtain another vector of the same size; $\theta=\{W_L, W_l,b_l\}_{l=1}^{L-1}$ is the set of all free parameters.

Given a training dataset $\{(\bx_n,\bal_n)\}_{n=1}^N$, where $\bx_n\in\mathbb{R}^d$ is a (column) feature vector and $\bal_n\in\mathbb{R}^J$ is the (column) one-hot label vector of $\bx_n$. For simplicity, denote 
$$X:=[\bx_1~\ldots~\bx_N]\in\mathbb{R}^{d\times N},\quad A:=[\bal_1~\ldots~\bal_N]\in\mathbb{R}^{J\times N},$$
$$\phi_{\mathrm{FNN}}(X;\theta):=[\phi_{\mathrm{FNN}}(\bx_1;\theta)~\ldots~\phi_{\mathrm{FNN}}(\bx_N;\theta)].$$ 
Consider the cross-entropy loss defined by \eqref{eq:sec2-03} and let $Z=\phi_{\mathrm{FNN}}(X;\theta)$. The cross-entropy model for the FNNs can be rewritten as
\begin{multline}\label{loss-FNN}
\min_{\theta}\mL(\phi_{\mathrm{FNN}}(X;\theta),A)=\frac{1}{N}\Big\|-\diag_v(A^\top(W_L\sigma(W_{L-1}\sigma(\cdots\sigma(W_1X+b_1\Bt)\\
\cdots)+b_{L-1}\Bt)))+\ln\sum_{k=1}^Je^{W_L(k,:)\sigma(W_{L- 1}\sigma(\cdots\sigma(W_1X+b_1\Bt)\cdots)+b_{L-1}\Bt)}\Big\|_1,
\end{multline}
where $\|\cdot\|_1$ is the $\ell^1$-norm of a vector. Here, $\diag_v(K)$ denotes the row vector formed by the diagonal entries of a square matrix $K\in\mathbb{R}^{N\times N}$.

For the cross-entropy model, training the deep network $\phi_{\mathrm{FNN}}$ is challenging due to the highly nonconvex and nested structure of the loss, which may lead to slow convergence, poor local minima, and vanishing gradients. To address this difficulty, we propose a layer separation formulation for the cross-entropy loss by introducing auxiliary variables that decouple the deep network into a sequence of shallow networks. This reformulation provides a new optimization framework that facilitates theoretical analysis and algorithmic design. Moreover, according to Remark \ref{rem1}, the cross-entropy loss can be controlled by the discrepancy between two feature representations, which gives the key theoretical justification for the proposed layer separation model.

\subsection{Layer separation model with auxiliary variables}

We consider the FNN defined by \eqref{eq:phi-FNN}. Our goal is to construct a loss with layer-separated neural networks that bounds the original cross-entropy loss. To separate the layers, we introduce the auxiliary variables
\begin{equation}\label{eq:variables_c}
c_l=
\begin{cases}
W_1X+b_1\Bt,\quad &l=1,\\
W_l\sigma(c_{l-1})+b_l\Bt,\quad &l=2,\ldots,L-1.
\end{cases}
\end{equation}
By these variables, we can decouple the deep network into a sequence of shallow networks. Specifically, we introduce $\{c_l\}$ in the optimization \eqref{loss-FNN} by adding $\ell^2$ penalty terms of the constraints \eqref{eq:variables_c}, each of which is for one layer of the neural network. This leads to the following layer separation model:
\begin{equation}\label{LySep_FNN}
\min_{W_l,b_l,c_l}\mLs^{\mathrm{FNN}}:=\frac{1}{\sqrt{N}}S(\Theta)^{\frac 12},
\end{equation}
where $\Theta=\{W_l,b_l,c_l\}$ collects all parameters,
\begin{multline*}
S(\Theta)=\Big\|-\diag_v(A^\top(W_L\sigma(c_{L-1})))+\ln\sum_{k=1}^Je^{W_L(k,:)\sigma(c_{L-1})}\Big\|_2^2\\
+\sum_{l=2}^{L-1}\omega_l\|W_l\sigma(c_{l-1})+b_l\Bt-c_l\|_\tF^2+\omega_1\|W_1X+b_1\Bt-c_1\|_\tF^2,
\end{multline*}
and $\{\omega_l\}_{l=1}^{L-1}$ are the weights of these constraints. If the weights are set to constants, the layer separation model will be inconsistent with the original cross-entropy model; namely, a set of $\{W_l,b_l,c_l\}$ may lead to small $\mLs^{\mathrm{FNN}}$ but significantly large $\mL$. Thus, we choose adaptive weights $\omega_l$ that depend explicitly on $\{W_l\}$. Here, we set
\begin{equation}\label{eq:weights}
\omega_l=\prod_{j=l+1}^{L}\|W_j\|_\tF^2,\quad \text{for}~l=1,\ldots,L-1.
\end{equation}
This particular choice of weights is motivated by the estimate established in Theorem \ref{thm3.1}.

The following result shows that the layer separation loss $\mLs^{\mathrm{FNN}}$ with the adaptive weights \eqref{eq:weights} yields an upper bound on the original cross-entropy loss. This provides a theoretical justification for using $\mLs^{\mathrm{FNN}}$ as a surrogate loss.

\begin{theorem}\label{thm3.1}
Suppose $\sigma$ is Lipschitz continuous, i.e., $|\sigma(z_1)-\sigma(z_2)|\leq B|z_1-z_2|$ for some $B>0$ and any $z_1,z_2\in\mathbb{R}$. Then for any $\{W_l\}_{l=1}^{L}$ and $\{b_l,c_l\}_{l=1}^{L-1}$, there holds
\begin{equation}\label{eq:thm3.1-00}
\mL(\phi_{\mathrm{FNN}},A)\leq C_B \cdot \sqrt{L-1} \cdot \mLs^{\mathrm{FNN}},
\end{equation}
where $C_{B}=\max\{\sqrt{2},2B,2B^{L-1}\}$.
\end{theorem}
\begin{proof}
According to \eqref{eq:sec2-07}, we obtain
\begin{multline}\label{eq:thm3.1-01}
\mL^2(\phi_{\mathrm{FNN}},A)\leq\frac{1}{N}\Big(2\Big\|-\diag_v(A^\top(W_L\sigma(c_{L-1})))+\ln\sum_{k=1}^Je^{W_L(k,:)\sigma(c_{L-1})}\Big\|_2^2\\
+4\|\phi_{\mathrm{FNN}}-W_L\sigma(c_{L-1})\|_\tF^2\Big).
\end{multline}
By the Lipschitz continuity of $\sigma$, it holds that
\begin{equation*}
\|\phi_{\mathrm{FNN}}-W_L\sigma(c_{L-1})\|_\tF\leq B\|W_L\|_\tF\|\mathcal{T}_{L-1}\|_\tF,
\end{equation*}
where $\mathcal{T}_{L-1}=W_{L-1}\sigma(\ldots\sigma(W_1X+b_1\Bt)\ldots)+b_{L-1}\Bt-c_{L-1}$. Then, using the triangle inequality and Lipschitz continuity again, we have
\begin{equation*}
\|\mathcal{T}_{L-1}\|_\tF
\leq B\|W_{L-1}\|_\tF\|\mathcal{T}_{L-2}\|_\tF+\|W_{L-1}\sigma(c_{L-2})+b_{L-1}\Bt-c_{L-1}\|_\tF,
\end{equation*}
where $\mathcal{T}_{L-2}=W_{L-2}\sigma(\ldots\sigma(W_1X+b_1\Bt)\ldots)+b_{L-2}\Bt-c_{L-2}$. 
Doing the preceding steps recursively, we finally have
\begin{multline}\label{eq:thm3.1-02}
\|\phi_{\mathrm{FNN}}-W_L\sigma(c_{L-1})\|_\tF\leq B^{L-1}\left(\prod_{k=2}^{L}\|W_k\|_\tF\right)\|W_1X+b_1\Bt-c_1\|_\tF\\
+\sum_{l=2}^{L-1}B^{L-l}\left(\prod_{k=l+1}^{L}\|W_k\|_\tF\right)\|W_l\sigma(c_{l-1})+b_l\Bt-c_l\|_\tF.
\end{multline}
Combining \eqref{eq:thm3.1-01} and \eqref{eq:thm3.1-02} yields \eqref{eq:thm3.1-00}.
\end{proof}

By Theorem \ref{thm3.1}, among all choices of network parameters and auxiliary variables, the layer separation loss $\mLs^{\mathrm{FNN}}$ is always an upper bound of the original cross-entropy loss $\mL(\phi_{\mathrm{FNN}},A)$ up to a constant. In particular, during the optimization process, if we make $\mLs^{\mathrm{FNN}}\to 0$, then $\mL(\phi_{\mathrm{FNN}},A)\to 0$.

\subsection{Algorithm for the FNN}\label{Algorithm:FNN}

In this subsection, we develop an alternating minimization algorithm to minimize $\mLs^{\mathrm{FNN}}$. At each step, only one parameter (matrix or vector) is updated, while the others are kept fixed. Since minimizing $\mLs^{\mathrm{FNN}}(\Theta)$ is equivalent to minimizing $S(\Theta)$, we work with $S(\Theta)$ in the sequel. Without ambiguity, for any $x\in\Theta$, we use $S(x)$ to denote the function $S$ with parameter $x$ varying and the others fixed in this section.  

We first consider updating $W_l$ with the other parameters fixed, for $l=1,\ldots, L-1$. In this case, minimizing $S(W_l)$ is equivalent to minimizing
\begin{equation}\label{eq:sec3.2-01}
\|W_lV_l-B_l\|_\tF^2+\lambda_l\|W_l\|_\tF^2,
\end{equation}
where
\begin{equation*}
V_l=\begin{cases}
\sigma(c_{l-1}),& 2\leq l\leq L-1,\\
X,& l=1,
\end{cases}\qquad
B_l=c_l-b_l\Bt, \quad 1\leq l\leq L-1,
\end{equation*}
and $\lambda_1=0$, $\lambda_2=\|W_{1}V_{1}-B_{1}\|_\tF^2$, $\lambda_l=\sum_{i=1}^{l-2}\left(\prod_{j=i+1}^{l-1}\|W_j\|_\tF^2\right)\|W_iV_i-B_i\|_\tF^2+\|W_{l-1}V_{l-1}-B_{l-1}\|_\tF^2$ ($3\leq l\leq L-1$)
are all fixed. The minimizer of \eqref{eq:sec3.2-01} is precisely the least-squares solution of the linear system
\begin{equation}\label{eq:wl}
W_l\begin{bmatrix}V_l & \sqrt{\lambda_l}I\end{bmatrix}=\begin{bmatrix}B_l & O\end{bmatrix},
\end{equation}
where $I$ is the identity matrix and $O$ is the zero matrix.

Similarly, $b_l$ is updated as the least-squares solution of the linear system 
\begin{equation}\label{eq:bl}
b_l\Bt=c_l-W_lV_l,\quad 1\leq l\leq L-1.
\end{equation}
Note that the least-squares solution is given by the mean of the column vectors on the right-hand side.

Furthermore, due to the nonlinearity of the activation function $\sigma$, the minimization sub-problems for $\{c_l\}$ cannot be solved in closed form. For $1\leq l\leq L-2$, each $c_l$ is involved only in two terms of $S(c_l)$, and reducing $S(c_l)$ is equivalent to reducing
\begin{equation*}
\omega_{l+1}\|W_{l+1}\sigma(c_l)-B_{l+1}\|_\tF^2+\omega_{l}\|W_lV_l+b_l\Bt-c_l\|_\tF^2.
\end{equation*}
The gradient of $S(c_l)$ is
\begin{equation}\label{eq:sec3.2-02}
\nabla_{c_l} S(c_l)=2\Big[\omega_{l+1}\left(W_{l+1}^\top(W_{l+1}\sigma(c_l)-B_{l+1})\right)\odot\sigma'(c_l)+\omega_{l}(c_l-W_lV_l-b_l\Bt)\Big].
\end{equation}
Here $\odot$ means entry-wise multiplication of two matrices of the same size. We update $\{c_l\}$ by gradient descent with a monotone backtracking line search: starting from $\tau>0$, we shrink $\tau\leftarrow \eta\tau$ ($0<\eta<1$) until 
\begin{equation*}
S\big(c_l-\tau \nabla_{c_l}S(c_l)\big)\leq S(c_l)
\end{equation*}
holds, and then set $c_l\leftarrow c_l-\tau \nabla_{c_l}S(c_l)$. 

Next, we consider the update of $c_{L-1}$ and $W_L$. Define the softmax probability matrix
\begin{equation*}
\mathcal{P}(W_L,c_{L-1})=
\begin{bmatrix}
p_1(W_L,c_{L-1})\\
\vdots\\
p_J(W_L,c_{L-1})
\end{bmatrix}
\quad \text{with} \quad
p_j(W_L,c_{L-1})=\frac{e^{W_L(j,:)\sigma(c_{L-1})}}{\sum_{k=1}^Je^{W_L(k,:)\sigma(c_{L-1})}},
\end{equation*}
where the division is entry-wise. Also, we define
\begin{equation*}
\mathcal{R}(W_L,c_{L-1})=-\diag_v(A^\top(W_L\sigma(c_{L-1})))+\ln\sum_{k=1}^Je^{W_L(k,:)\sigma(c_{L-1})}.
\end{equation*}
To simplify the notations, we write
\begin{equation}\label{eq:sec3.2-03}
\begin{aligned}
\mathcal{R}(c_{L-1}) := \mathcal{R}(W_L,c_{L-1}),\quad 
\mathcal{P}(c_{L-1}):=\mathcal{P}(W_L,c_{L-1}) \quad \text{when $W_L$ is fixed},\\
\mathcal{R}(W_L) := \mathcal{R}(W_L,c_{L-1}),\quad
\mathcal{P}(W_L):=\mathcal{P}(W_L,c_{L-1}) \quad \text{when $c_{L-1}$ is fixed}.
\end{aligned}
\end{equation}
Then, for the variable $c_{L-1}$, we reduce $S(c_{L-1})$ by reducing
\begin{equation*}
\|\mathcal{R}(c_{L-1})\|_2^2+\|W_L\|_\tF^2\|W_{L-1}V_{L-1}+b_{L-1}\Bt-c_{L-1}\|_\tF^2.
\end{equation*}
The gradient of $S(c_{L-1})$ is
\begin{multline}\label{eq:sec3.2-4}
\nabla_{c_{L-1}}S(c_{L-1})=2(W_L^\top(\mathcal{P}(c_{L-1})-A)\diag(\mathcal{R}(c_{L-1})))\odot\sigma'(c_{L-1})\\
+2\|W_L\|_\tF^2(c_{L-1}-W_{L-1}V_{L-1}-b_{L-1}\Bt).
\end{multline}
Here $\diag(v)\in\mathbb{R}^{N\times N}$ denotes the diagonal matrix generated by the vector $v\in\mathbb{R}^{N}$. And for the variable $W_L$, reducing $S(W_L)$ is equivalent to reducing
\begin{equation*}
\lambda\|W_L\|_\tF^2+\|\mathcal{R}(W_L)\|_2^2,
\end{equation*}
where $\lambda=\|W_{L-1}V_{L-1}-B_{L-1}\|_\tF^2+\sum_{l=1}^{L-2}\Big(\prod_{j=l+1}^{L-1}\|W_j\|_\tF^2\Big)\|W_lV_{l}-B_l\|_\tF^2$. The gradient of $S(W_L)$ is
\begin{equation}\label{eq:sec3.2-5}
\nabla_{W_L}S(W_L)=2\lambda W_L+2(\mathcal{P}(W_L)-A)\diag(\mathcal{R}(W_L))\sigma(c_{L-1})^{\top}.
\end{equation}

In summary, we propose the algorithm that updates $\{W_l, b_l, c_l\}$ for every $l$ from large to small (see Algorithm \ref{alg01}). Here, $\tau_l^k>0$ is the learning rate in the gradient descent update, selected by backtracking line search.

\begin{algorithm}
\DontPrintSemicolon
\KwIn{input data $X$; label matrix $A$; maximum number of iterations $N_k$}
\KwOut{A feasible solution $\{W_l,b_l,c_l\}$.}
\Begin{Initialize $W_L$ and $\{W_l,b_l,c_l\}_{l=1}^{L-1}$\;
\For{$k = 0,\cdots,N_k-1$}{
$W_L\leftarrow W_L-\tau_L^k\nabla_{W_L}S(W_L)$,\;
\For{$l = L-1,\cdots,1$}{
$c_l\leftarrow c_l-\tau_l^k\nabla_{c_l}S(c_l)$,\;
Solve $W_l\begin{bmatrix}V_l & \sqrt{\lambda_l}I\end{bmatrix}=\begin{bmatrix}B_l & O\end{bmatrix}$ for $W_l$\;
Solve $b_l\Bt=c_l-W_lV_l$ for $b_l$\;}}
return $\{W_l,b_l,c_l\}$}
\caption{Solve $\min_{W_l,b_l,c_l} \mLs^{\mathrm{FNN}}$\label{alg01}}
\end{algorithm}

\subsection{Analysis}
Next, we investigate the decreasing property of Algorithm~\ref{alg01}. We first prove a basic inequality for functions with Lipschitz continuous gradients.

\begin{lemma}\label{lem3.2}
Let $f:\mathbb{R}^{m\times n}\to\mathbb{R}$ be continuously differentiable. Assume the gradient of $f$ is Lipschitz continuous with constant $C>0$. Then, for all $X,Y\in\mathbb{R}^{m\times n}$, it holds that
\begin{equation}\label{eq:lem3.2-1}
f(Y)\leq f(X)+\langle \nabla f(X),\,Y-X\rangle_\tF +\frac{C}{2}\|Y-X\|_\tF^2,
\end{equation}
where $\langle \cdot,\cdot\rangle_\tF$ is the Frobenius inner product.
\end{lemma}
\begin{proof}
For $X,Y\in\mathbb{R}^{m\times n}$, define the function
\begin{equation*}
\varphi(t)=f(X+t(Y-X)),\quad t\in[0,1],
\end{equation*}
which satisfies $\varphi(0)=f(X)$ and $\varphi(1)=f(Y)$. Since $f$ is continuously differentiable, we can use the chain rule and obtain
\begin{equation}\label{eq:lem3.2-2}
\varphi'(t)=\left\langle \nabla f(X+t(Y-X)),Y-X\right\rangle_\tF.
\end{equation}
Integrating \eqref{eq:lem3.2-2} from $t=0$ to $1$ yields
\begin{multline}\label{eq:lem3.2-3}
f(Y)-f(X)=\int_0^1 \left\langle \nabla f(X+t(Y-X)),Y-X\right\rangle_\tF \dt\\
=\left\langle \nabla f(X),Y-X\right\rangle_\tF
+\int_0^1 \left\langle \nabla f(X+t(Y-X))-\nabla f(X),Y-X\right\rangle_\tF \dt.
\end{multline}
Applying the Cauchy-Schwarz inequality to the integral term and using the Lipschitz continuity of $\nabla f$, we obtain
\begin{equation}\label{eq:lem3.2-4}
\begin{aligned}
&\int_0^1 \left\langle \nabla f(X+t(Y-X))-\nabla f(X),Y-X\right\rangle_\tF \dt\\
\leq&\int_0^1 \|\nabla f(X+t(Y-X))-\nabla f(X)\|_\tF\|Y-X\|_\tF \dt\\
\leq&C\|Y-X\|_\tF^2\int_0^1 t \dt=\frac 12 C\|Y-X\|_\tF^2.
\end{aligned}
\end{equation}
Substituting \eqref{eq:lem3.2-4} into \eqref{eq:lem3.2-3} yields \eqref{eq:lem3.2-1}.
\end{proof}

\begin{lemma}\label{lem3.3}
Suppose $\sigma\in C^2(\mathbb R)$ and $|\sigma(z)|\leq M_0$, $|\sigma'(z)|\leq M_1$ and $|\sigma''(z)|\leq M_2$, $\forall z\in\mathbb{R}$, for some $M_0$, $M_1$, $M_2>0$. Let $R>0$ and assume that $\{W_l\}_{l=1}^L$ and $\{b_l,c_l\}_{l=1}^{L-1}$ are bounded by $R$ in Frobenius norm. Then there exist constants $C_l>0$, $l=1,\dots,L$, such that\\
(i) For each $l=1,\dots,L-1$, the gradient $\nabla_{c_l} S(c_l)$ defined by \eqref{eq:sec3.2-02} and \eqref{eq:sec3.2-4} is Lipschitz continuous, i.e.,
\begin{equation}\label{lem3.3-1}
\|\nabla_{c_l} S(\hat c_l)-\nabla_{c_l} S(\tilde c_l)\|_\tF
\leq C_l\|\hat c_l-\tilde c_l\|_\tF,
\qquad \forall \hat c_l, \tilde c_l\in\mathbb{R}^{M\times N},
\end{equation}
(ii) The gradient $\nabla_{W_L} S(W_L)$ defined by \eqref{eq:sec3.2-5} is Lipschitz continuous, i.e.,
\begin{equation}\label{eq:lem3.3-2}
\|\nabla_{W_L} S(\hat W_L)-\nabla_{W_L} S(\tilde W_L)\|_\tF
\leq C_L\|\hat W_L-\tilde W_L\|_\tF,
\qquad \forall \hat W_L, \tilde W_L\in\mathbb{R}^{J\times M}.
\end{equation}
\end{lemma}

\begin{proof}
We first consider the case $1\leq l\leq L-2$. For convenience, define $G(c_l)=W_{l+1}^\top(W_{l+1}\sigma(c_l)-B_{l+1})$. According to \eqref{eq:sec3.2-02}, for all $\hat c_l,\tilde c_l\in\mathbb{R}^{M\times N}$, we have
\begin{equation}\label{eq:lem3.3-3}
\|\nabla_{c_l} S(\hat c_l)-\nabla_{c_l} S(\tilde c_l)\|_\tF\leq2(\omega_{l+1}\|G(\hat c_l)\odot\sigma'(\hat c_l)-G(\tilde c_l)\odot\sigma'(\tilde c_l)\|_\tF
+\omega_{l}\|\hat c_l-\tilde c_l\|_\tF).
\end{equation}
According to the boundedness of $\sigma'$ and $\sigma''$, using the mean value theorem entrywise yields
\begin{equation}\label{eq:lem3.3-4}
\begin{aligned}
&\|G(\hat c_l)\odot\sigma'(\hat c_l)-G(\tilde c_l)\odot\sigma'(\tilde c_l)\|_\tF\\
\leq&\|G(\hat c_l)-G(\tilde c_l)\|_\tF\|\sigma'(\hat c_l)\|_\infty+\|G(\tilde c_l)\|_\tF\|\sigma'(\hat c_l)-\sigma'(\tilde c_l)\|_\tF\\
\leq& M_1\|G(\hat c_l)-G(\tilde c_l)\|_\tF+M_2\|G(\tilde c_l)\|_\tF\|\hat c_l-\tilde c_l\|_\tF\\
\leq&(M_1^2\|W_{l+1}\|_\tF^2+M_0M_2\sqrt{MN}\|W_{l+1}\|_\tF^2+M_2\|W_{l+1}^\top B_{l+1}\|_\tF)\|\hat c_l-\tilde c_l\|_\tF.
\end{aligned}
\end{equation}
Here $\|\cdot\|_\infty$ is the entrywise max norm. Combining \eqref{eq:lem3.3-3}-\eqref{eq:lem3.3-4}, we obtain \eqref{lem3.3-1} for $l=1,\ldots,L-2$.

The proof for $S(c_{L-1})$ is similar. For simplicity, let $c:=c_{L-1}$ and
\begin{equation*}
H(c):=W_L^\top\big(\mathcal{P}(c)-A\big)\diag(\mathcal{R}(c)),
\end{equation*}
where $\mathcal{P}(c)$ and $\mathcal{R}(c)$ are defined by \eqref{eq:sec3.2-03}. For any $\hc,\tilde c\in\mathbb{R}^{M\times N}$, according to \eqref{eq:lem1-3}, we obtain that the softmax map $p$: $\mathbb {R}^J\to\mathbb {R}^J$ is 1-Lipschitz in $\ell_2$. Thus,
\begin{equation}\label{eq:lem3.3-5}
\|\mathcal{P}(\hat c)-\mathcal{P}(\tilde c)\|_\tF
\leq \|W_L\sigma(\hat c)-W_L\sigma(\tilde c)\|_\tF
\leq M_1\|W_L\|_\tF\|\hat c-\tilde c\|_\tF.
\end{equation}
Since each column of $\mathcal{P}(c)$ is a probability vector and each column of $A$ is a one-hot vector, according to \eqref{eq:lem1-4}, it implies
\begin{equation}\label{eq:lem3.3-6}
\|\mathcal{P}(c)-A\|_\tF^2=\sum_{n=1}^N \|p_n-\bal_n\|_2^2 \leq 2N.
\end{equation}
Moreover, note that $\mathcal R(c) = \mL_{\mathrm{vec}}(W_L\sigma(c),A)\in\mathbb R^N$, using Lemma \ref{lem:2.1} and the boundedness of $\sigma$ and $\sigma'$, we can derive
\begin{equation}\label{eq:lem3.3-7}
\|\mathcal{R}(\hc)-\mathcal{R}(\tilde c)\|_2\leq\sqrt{2}\|W_L(\sigma(\hc)-\sigma(\tilde c))\|_\tF\leq\sqrt{2}M_1\|W_L\|_\tF\|\hc-\tilde c\|_\tF,
\end{equation}
and
\begin{equation}\label{eq:lem3.3-8}
\|\mathcal{R}(c)\|_\infty\leq2\sqrt{M}\|W_L\|_\tF\|\sigma(c)\|_\infty+\ln J\leq2 \sqrt{M}M_0\|W_L\|_\tF+\ln J.
\end{equation}
Then, using \eqref{eq:lem3.3-7}-\eqref{eq:lem3.3-6}, it follows that
\begin{equation}\label{eq:lem3.3-9}
\begin{aligned}
&\|H(\hc)-H(\tilde c)\|_\tF\\
=&\|W_L^\top\left((\mathcal{P}(\hc)-\mathcal{P}(\tilde c))\diag(\mathcal{R}(\hc))+(\mathcal{P}(\tilde c)-A)\diag(\mathcal{R}(\hc)-\mathcal{R}(\tilde c))\right)\|_\tF\\
\leq&\|W_L\|_\tF\left(\|\mathcal{P}(\hc)-\mathcal{P}(\tilde c)\|_\tF\|\mathcal{R}(\hc)\|_\infty+\|\mathcal{P}(\tilde c)-A\|_\tF\|\mathcal{R}(\hc)-\mathcal{R}(\tilde c)\|_2\right)\\
\leq&M_1\|W_L\|_\tF^2(2\sqrt{M}M_0\|W_L\|_\tF+\ln J+2\sqrt{N})\|\hc-\tilde c\|_\tF,
\end{aligned}
\end{equation}
and
\begin{equation}\label{eq:lem3.3-10}
\|H(c)\|_\infty\leq\|W_L\|_\tF\|\mathcal{P}(c)-A\|_\tF\|\mathcal{R}(c)\|_\infty\leq\sqrt{2N}\|W_L\|_\tF(2 \sqrt{M}M_0\|W_L\|_\tF+\ln J).
\end{equation}
Using the expression of $\nabla_{c_{L-1}} S(c_{L-1})$ and \eqref{eq:lem3.3-9}-\eqref{eq:lem3.3-10}, we obtain
\begin{equation*}
\begin{aligned}
&\|\nabla_{c_{L-1}} S(\hc)-\nabla_{c_{L-1}} S(\tilde c)\|_\tF\\
\leq&2\|H(\hc)\odot\sigma'(\hc)-H(\tilde c)\odot\sigma'(\tilde c)\|_\tF+2\|W_L\|_\tF^2\|\hc-\tilde c\|_\tF\\
\leq& 2M_1\|H(\hc)-H(\tilde c)\|_\tF+2M_2\|H(\tilde c)\|_\infty\|\hc-\tilde c\|_\tF+2\|W_L\|_\tF^2\|\hc-\tilde c\|_\tF\\
\leq& C_{L-1}\|\hc-\tilde c\|_\tF,
\end{aligned}
\end{equation*}
where $C_{L-1}=2\|W_L\|_\tF\Big((M_1^2\|W_L\|_\tF+\sqrt{2N}M_2)(2\sqrt{M}M_0\|W_L\|_\tF+\ln J)+\|W_L\|_\tF(1+2\sqrt{N}M_1^2)\Big)$.

Finally, for $W_L$, similar to \eqref{eq:lem3.3-5}-\eqref{eq:lem3.3-8}, according to the boundedness of $\sigma$, Lemma~\ref{lem:2.1} and \eqref{eq:lem1-3}, we obtain that $\mathcal R(W_L)$ and $\mathcal{P}(W_L)$ are Lipschitz continuous and bounded. Therefore,
\begin{equation*}
\|\nabla_{W_L} S(\hat W_L)-\nabla_{W_L} S(\tilde W_L)\|_\tF\leq C_L\|\hat W_L-\tilde W_L\|_\tF,
\end{equation*}
where $C_L=2\lambda+2NMM_0^2(2\sqrt{M}M_0R+\ln J+2)$.
\end{proof}

Next, we present the decreasing loss property of the alternating minimization algorithm.

\begin{theorem}\label{thm3.4}
Under the hypothesis of Lemma \ref{lem3.3}, let $\{C_l\}_{l=1}^{L}$ be the Lipschitz constants given in Lemma \ref{lem3.3}. If the learning rates $\tau_l^k\in(0,\frac{2}{C_l}]$, it holds
\begin{equation}\label{eq:thm3.4-0}
S(\Theta^{k+1})\leq S(\Theta^k),\qquad \forall k\geq 0,
\end{equation}
where $\Theta^k=\{W_L^k,W_l^k,b_l^k,c_l^k\}_{l=1}^{L-1}$ is the set of parameters after the $k$-th update in Algorithm \ref{alg01}.
\end{theorem}
\begin{proof}
Fix $k\geq 0$. By Lemma~\ref{lem3.3}(ii), $\nabla_{W_L} S(W_L)$ is Lipschitz continuous with constant $C_L$. Updating $W_L$ by gradient descent yields
\begin{equation*}
W_L^{k+1}=W_L^k-\tau_{L}^k\nabla_{W_L} S(W_L^k).
\end{equation*}
Then, according to Lemma \ref{lem3.2}, we obtain
\begin{multline*}
S(W_L^{k+1})\leq S(W_L^k)-\tau_{L}^k\|\nabla_{W_L} S(W_L^k)\|_\tF^2+\frac{C_L}{2}(\tau_{L}^k)^2\|\nabla_{W_L} S(W_L^k)\|_\tF^2\\
=S(W_L^k)-\tau_{L}^k\Big(1-\frac{C_L}{2}\tau_{L}^k\Big)\|\nabla_{W_L} S(W_L^k)\|_\tF^2.
\end{multline*}
Because $0<\tau_{L}^k\leq \frac{2}{C_L}$, it follows that
\begin{equation}\label{eq:thm3.4-1}
S(W_L^{k+1})\leq S(W_L^k).
\end{equation}

Next, for $l=1,\dots,L-1$, recall that $c_l$ is updated by
\begin{equation*}
c_l^{k+1}:=c_l^{k}-\tau_{l}^k\nabla_{c_l} S(c_l^k).
\end{equation*}
By a similar argument, using Lemma \ref{lem3.2} and $0<\tau_{l}^k\leq \frac{2}{C_l}$, we obtain
\begin{equation}\label{eq:thm3.4-2}
S(c_l^{k+1})\leq S(c_l^k)-\tau_{l}^k\Big(1-\frac{C_l}{2}\tau_{l}^k\Big)
\|\nabla_{c_l} S(c_l^k)\|_\tF^2\leq S(c_l^k).
\end{equation}

For $\{W_l,b_l\}_{l=1}^{L-1}$, they are updated as the least squares solutions of the linear systems \eqref{eq:wl}-\eqref{eq:bl}. So,
\begin{equation}\label{eq:thm3.4-3}
S(W_l^{k+1})\leq S(W_l^{k}) \quad \text{and} \quad S(b_l^{k+1})\leq S(b_l^{k}).
\end{equation}

Finally, combining \eqref{eq:thm3.4-1}-\eqref{eq:thm3.4-3}, we conclude \eqref{eq:thm3.4-0}
\end{proof}

\section{Convolutional neural network}\label{sec:3}
In this section, we consider the CNN case, where the neurons are formulated by convolutions and pooling instead of dense matrices. For simplicity, we describe the single-channel case, i.e., the feature maps in the input and hidden layers are single-channel; the multi-channel case is analogous and will be presented in the numerical experiments.

Suppose $X\in\mathbb{R}^{m\times n}$ is an input sample. Then $\mathrm{vec}(X)\in\mathbb{R}^{d}$, where $d=mn$. Assume that the CNN consists of $L_1$ convolution-pooling layers and $L_2$ fully connected layers, i.e.,
\begin{align}\label{eq:phi-CNN}
\phi_{\mathrm{CNN}}(X;\theta):=\hW_{L_2}\sigma(\hW_{L_2-1}\sigma(\cdots\sigma(\hW_1P_{L_1}\sigma(K_{L_1}\sigma(\cdots P_1\sigma(K_1\mathrm{vec}(X)+b_1)\nonumber\\
\cdots)+ b_{L_1})+ \hb_{1}\Big)
\cdots)+ \hb_{L_2-1}),
\end{align}
where $\{K_l\}_{l=1}^{L_1}$ and $\{P_l\}_{l=1}^{L_1}$ are sparse matrices determined by the convolution operator and the pooling operator; $\{b_l\}_{l=1}^{L_1}$ are the corresponding bias terms for the convolutional layers; $\{\hW_l\}_{l=1}^{L_2}$ are the weights of the fully connected layers; $\{\hb_l\}_{l=1}^{L_2-1}$ are the corresponding biases; $\theta=\{K_l,b_l,\hW_l,\hb_l\}$ is the set of all free parameters.

Given a training dataset $\{(X_n,\bal_n)\}_{n=1}^N$, where $X_n\in\mathbb{R}^{m\times n}$ is the input sample and $\bal_n\in\mathbb{R}^J$ is the corresponding one-hot label of $X_n$. For simplicity, denote 
$$\hat X:=[\mathrm{vec}(X_1)~\ldots~\mathrm{vec}(X_N)],\quad \phi_{\mathrm{CNN}}(\hat X;\theta):=[\phi_{\mathrm{CNN}}(X_1;\theta)~\ldots~\phi_{\mathrm{CNN}}(X_N;\theta)].$$
Consider the cross-entropy loss defined by \eqref{eq:sec2-03} and let $Z=\phi_{\mathrm{CNN}}(\hat X;\theta)$. The cross-entropy model for CNN can be rewritten as follows:
\begin{equation}\label{loss-CNN}
\min_{\theta}\mL(\phi_{\mathrm{CNN}}(\hat X;\theta),A):=\frac{1}{N}\Big\|-\diag_v(A^\top(\phi_{\mathrm{CNN}}(\hat X;\theta)))+\ln\sum_{k=1}^Je^{\phi_{\mathrm{CNN}}(\hat X;\theta)(k,:)}\Big\|_1.
\end{equation}

\subsection{Layer separation model with auxiliary variables}
To separate the layers of the CNN, we introduce the auxiliary variables
\begin{equation*}
c_l=
\begin{cases}
K_lP_{l-1}\sigma(c_{l-1})+b_l\Bt,\quad l=2,\ldots,L_1,\\
K_1\hat X+b_1\Bt,\quad l=1,
\end{cases}
\end{equation*}
and
\begin{equation*} 
\hc_l=
\begin{cases}
\hW_l\sigma(\hc_{l-1})+\hb_l\Bt,\quad l=2,\ldots,L_2-1,\\
\hW_1P_{L_1}\sigma(c_{L_1})+\hb_1\Bt,\quad l=1.
\end{cases}
\end{equation*}
Then the layer separation model is given by
\begin{equation}\label{LySep-loss-CNN}
\min_{\Theta}\mLs^{\mathrm{CNN}}:=\frac{1}{\sqrt{N}}T(\Theta)^{\frac 12}, 
\end{equation}
where $\Theta=\{K_l,b_l,c_l,\hW_l,\hb_l,\hc_l\}$ collects all parameters, and
\begin{multline*}
T(\Theta)=\Big\|-\diag_v(A^\top(\hW_{L_2}\sigma(\hc_{L_2-1})))+\ln\sum_{k=1}^Je^{\hW_{L_2}(k,:)\sigma(\hc_{L_2-1})}\Big\|_2^2\\
+\sum_{l=2}^{L_1}\omega_l\|K_lP_{l-1}\sigma(c_{l-1})+b_l\Bt-c_l\|_\tF^2+\omega_1\|K_1\hat X+b_1\Bt-c_1\|_\tF^2\\
+\sum_{l=2}^{L_2-1}\hat{\omega}_l\|\hW_l\sigma(\hc_{l-1})+\hb_l\Bt-\hc_l\|_\tF^2+\hat{\omega}_1\|\hW_1P_{L_1}\sigma(c_{L_1})+\hb_1\Bt-\hc_1\|_\tF^2,
\end{multline*}
with specially designed adaptive weights 
\begin{align*}
&\omega_{L_1}=\prod_{k=2}^{L_2}\|\hW_k\|_\tF^2\|\hW_1P_{L_1}\|_\tF^2,\\
&\omega_l=\omega_{L_1}\prod_{j=l+1}^{L_1}\|K_jP_{j-1}\|_\tF^2,\quad \text{for}~ l=1,\ldots,L_1-1,\\
&\hat{\omega}_l=\prod_{k=l+1}^{L_2}\|\hW_k\|_\tF^2, \quad\text{for}~ l=1,\ldots,L_2-1.
\end{align*}

The following result guarantees that the original cross-entropy loss $\mL(\phi_{\mathrm{CNN}}, A)$ is always bounded above by the layer separation loss $\mLs^{\mathrm{CNN}}$. The proof is similar to that of Theorem \ref{thm3.1}, so we do not present it here. 

\begin{theorem}\label{thm4.1}
Suppose $\sigma$ is Lipschitz continuous with constant $B>0$. Let $\mL(\phi_{\mathrm{CNN}},A)$ and $\mLs^{\mathrm{CNN}}$ be defined by \eqref{loss-CNN} and \eqref{LySep-loss-CNN}, respectively. Then for all $\{K_l,b_l,c_l\}_{l=1}^{L_1}$ and $\{\hW_{L_2},\hW_l,\hb_l,\hc_l\}_{l=1}^{L_2-1}$, it holds that
\begin{equation}\label{eq:thm3.4-00}
\mL(\phi_{\mathrm{CNN}},A)\leq C_{B} \cdot \sqrt{L_1+L_2-1} \cdot \mLs^{\mathrm{CNN}},
\end{equation}
where $C_{B}=\max\{\sqrt{2},2B,2B^{L_1+L_2-1}\}$.
\end{theorem}

\subsection{Algorithm for the CNN}

In this subsection, we develop an alternating minimization algorithm to minimize $\mLs^{\mathrm{CNN}}$. The loss function is quadratic with respect to $\{K_l,b_l\}_{l=1}^{L_1}$ and $\{\hat W_l,\hat b_l\}_{l=1}^{L_2-1}$, which allows closed-form least-squares updates for these parameters. Also, due to the nonlinearity of the activation function $\sigma$, we update $\{\hW_{L_2},c_l,\hc_l\}$ by gradient descent with a backtracking line search. Since minimizing $\mLs^{\mathrm{CNN}}(\Theta)$ is equivalent to minimizing $T(\Theta)$, we work with $T(\Theta)$ in the sequel.

The gradient descent update for $\{\hW_{L_2},c_l,\hc_l\}$ and the closed-form update for $\{\hW_l,\hb_l,b_l\}$ are analogous to those in Section \ref{Algorithm:FNN}. For $K_l$, it corresponds to a convolution kernel, 
so $K_l$ has a parameter-sharing structure. In every iteration, we update $K_l$ with other parameters fixed by solving the linear systems
\begin{equation}\label{eq:4.2-1}
K_l\begin{bmatrix}U_l & \sqrt{\lambda_l}P_{l-1}\end{bmatrix}=\begin{bmatrix}B_l & O\end{bmatrix},
\end{equation}
where 
\begin{equation*}
U_l=\begin{cases}
P_{l-1}\sigma(c_{l-1}),& 2\leq l\leq L_1,\\
\hat X,& l=1,\end{cases}\qquad
B_l=c_l-b_l\Bt, \quad 1\leq l\leq L_1.
\end{equation*}
and $\lambda_1=0$, $\lambda_2=\|K_1U_1-B_1\|_\tF^2$, $
\lambda_l=\sum_{i=1}^{l-2}\left(\prod_{j=i+1}^{l-1}\|K_jP_{j-1}\|_\tF^2\right)\|K_iU_i-B_i\|_\tF^2+\|K_{l-1}U_{l-1}-B_{l-1}\|_\tF^2$($3\leq l\leq L_1-1$)
are all fixed.
Despite the parameter sharing, \eqref{eq:4.2-1} imposes linear constraints on the parameters of $K_l$, so \eqref{eq:4.2-1} can be reformulated as
a modified linear system
\begin{equation}\label{eq:4.2-2}
A_lk_l = \begin{bmatrix}B_l & O\end{bmatrix},
\end{equation}
where $A_l$ is a matrix composed of the elements in $[U_l~\sqrt{\lambda_l}P_{l-1}]$, and $k_l$ is the vector consisting of all independent parameters of $K_l$.

To sum up, we present the algorithm that updates $\{\hW_l,\hb_l, K_l,b_l,c_l,\hc_l\}$ alternatively for every $l$ from large to small (see Algorithm \ref{alg02}). 
\begin{algorithm}
\DontPrintSemicolon
\KwIn{input data $\hat X$; label $A$; number of iterations $N_k$}
\KwOut{a feasible solution $\{\hW_l,\hb_l, K_l,b_l,c_l,\hc_l\}$.}
\Begin{Initialize $\hW_{L_2}$, $\{\hW_l,\hb_l,\hc_l\}_{l=1}^{L_2-1}$ and $\{K_l,b_l,c_l\}_{l=1}^{L_1}$\;
\For{$k = 0,\cdots,N_k-1$}{
$\hW_{L_2}\leftarrow \hW_{L_2}-\hat\tau_{L_2}^k\nabla_{\hW_{L_2}}T(\Theta)$,\;
\For{$l = L_2-1,\cdots,1$}{
$\hc_l\leftarrow \hc_l-\hat\tau_l^k\nabla_{\hc_l}T(\Theta)$,\;
Solve the linear system for $\hW_l$ obtained in \eqref{LySep-loss-CNN}\;
Solve the linear system for $\hb_l$ obtained in \eqref{LySep-loss-CNN}\;}
\For{$l = L_1,\cdots,1$}{
$c_l\leftarrow c_l-\tau_l^k\nabla_{c_l}T(\Theta)$\;
Solve \eqref{eq:4.2-2} for $k_l$, and then recover $K_l$\;
Solve the linear system for $b_l$ obtained in \eqref{LySep-loss-CNN}\;}}
return $\{\hW_l,\hb_l, K_l,b_l,c_l,\hc_l\}$}
\caption{Solve $\min_{\{\hW_l,\hb_l, K_l,b_l,c_l,\hc_l\}} \mLs^{\mathrm{CNN}}$\label{alg02}}
\end{algorithm}

Next, we investigate the decreasing property of Algorithm~\ref{alg02}. The proof of the following theorem is analogous to that of the corresponding results in Section~\ref{Algorithm:FNN}. Since the only modification is replacing the fully connected linear block with the structured convolutional block $K_l$, we omit the repetitive details and state the results directly.

\begin{theorem}
Suppose $\sigma\in C^2(\mathbb R)$ and $\sigma$, $\sigma'$ and $\sigma''$ are bounded functions. Let $\hat R>0$ and assume that $\hW_{L_2}$, $\{\hW_l,\hb_l,\hc_l\}_{l=1}^{L_2-1}$ and $\{K_l,b_l,c_l\}_{l=1}^{L_1}$ are bounded by $\hat R$ in Frobenius norm. Then there exist constants $C_{l}>0$ for $l=1,\ldots,L_1$ and $\hat C_{l}>0$ for $l=1,\ldots,L_2$, depending on $\hat R$ and $\sigma$ such that\\
(i) For each $l=1,\dots,L_1$, $\nabla_{c_l} T(c_l)$ is Lipschitz continuous with constant $C_{l}$.\\
(ii) For each $l=1,\dots,L_2-1$, $\nabla_{\hc_l} T(\hc_l)$ is Lipschitz continuous with constant $\hat C_{l}$.\\
(iii) The gradient $\nabla_{\hW_{L_2}} T(\hW_{L_2})$ is Lipschitz continuous with constant $\hat C_{L_2}$.\\
Moreover, if the learning rates $\tau_l^k\in(0,\frac{2}{C_{l}}]$ and $\hat\tau_l^k\in(0,\frac{2}{\hat C_{l}}]$, it holds
\begin{equation*}
T(\Theta^{k+1})\leq T(\Theta^k),\qquad \forall k\geq 0,
\end{equation*}
where $\Theta^k=\{\hW_l^k,\hb_l^k, K_l^k,b_l^k,c_l^k,\hc_l^k\}$ is the set of parameters after the $k$-th update in Algorithm \ref{alg02}.
\end{theorem}
\section{Numerical Experiments}\label{sec:4}

In this section, we compare the standard cross-entropy models with the proposed layer separation models across three classification tasks. Our main objective is to assess the optimization behavior of the proposed methods under different network settings. Accordingly, we report the training cross-entropy loss values and classification accuracy throughout the iterative process.

The common experimental settings are summarized as follows.
\begin{itemize}
\item {\em Environment for cross-entropy model.}
Cross-entropy models are implemented in Python with the PyTorch library. The training objective is to reduce the cross-entropy loss. Standard gradient descent solves the optimization with fine-tuned learning rates.

\item {\em Environment for layer separation model.}
Layer separation models with FNNs and CNNs are solved by Algorithms \ref{alg01} and \ref{alg02}, respectively. They are implemented in Matlab, where the subroutines {\tt mldivide} (i.e., \textbackslash ) or {\tt mrdivide} (i.e., /) are called to solve the linear least squares subproblems arising in the alternating updates.

\item {\em Initialization of variables.}
The network parameters are randomly initialized using a common distribution in deep learning.

\item {\em Evaluation metrics.}
We evaluate model performance through classification accuracy
\begin{equation*}
\mathrm{Acc}=\frac{1}{N}\sum_{n=1}^{N}\mathbf{1}\{\hat{\bal}_n=\bal_n\},
\end{equation*}
where $\hat{\bal}_n$ and $\bal_n$ denote the predicted and true class labels, respectively.

\item {\em Randomness.}
To reduce the effect of initialization randomness, each experiment is repeated 10 times with different random seeds. This is equivalent to taking different initial guesses for the optimization. We present the mean and standard deviation (shown as ``mean $\pm$ standard deviation'') of the results of our numerical experiments in the tables below. 
\end{itemize}

Since the two models are implemented in different software environments, the comparisons below are intended to assess optimization behavior and stability with respect to random initialization, rather than raw wall-clock efficiency across implementations. For ease of presentation, throughout this section, CE-FNN and CE-CNN refer to the conventional softmax cross-entropy training formulations for FNNs and CNNs, respectively, whereas LySep-FNN and LySep-CNN refer to the corresponding layer-separation models, respectively.

\subsection{Binary classification with FNNs}
\subsubsection{Circle-in/out classification}
In this first example, we consider a circle-in/out classification problem on the domain $\Omega=[-1,1]^2$. The sample points are generated uniformly at random in $\Omega$ and are labeled according to whether they lie inside a circle centered at the origin. For a balanced classification, we set the circle radius $r=\sqrt{\frac{2}{\pi}}$ such that the area of the circle equals one-half of $\Omega$. The label is defined by $\bal_n=\mathbf{1}_{\{\|x_n\|_2\leq r\}}$. The training and testing sets contain 3000 and 1000 samples, respectively. The label distribution of the test set is shown in Figure~\ref{Fig_1}. To examine the influence of the network architecture, we test the following depth-width combinations $(L,M)=(3,10)$, $(3,16)$, $(3,20)$, $(10,10)$ and $(20,10)$ and use the $\tanh$ activation function. For each combination, we record the cross-entropy loss and training accuracy over $10^4$ iterations to compare the optimization behavior of the two models.
\begin{figure}[!ht]
\centering
\includegraphics[scale=0.14]{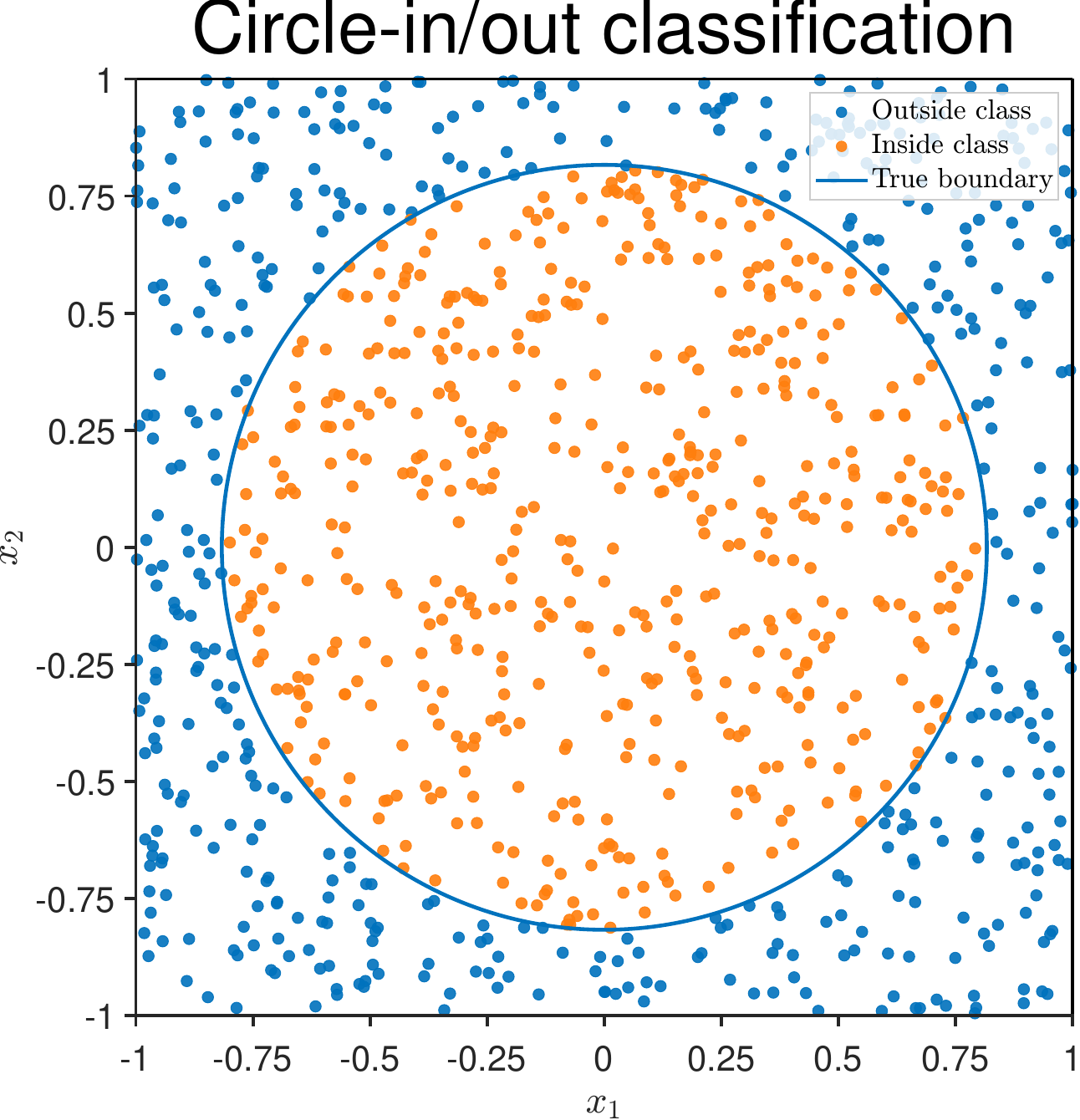}
\caption{\em Test set label distribution for the circle-in/out classification task in Example~1. The circle denotes the true decision boundary.}
\label{Fig_1}
\end{figure}

\begin{table}[!ht]
\fontsize{7}{10.5}\selectfont
\setlength{\tabcolsep}{6pt}
\centering
\begin{tabular}{clcc}
\toprule
$(L,M)$ & Models & Cross-entropy loss & Accuracy \\
\midrule
\multirow{2}{*}{$(3,10)$}
& CE-FNN    & 5.27e-02$\pm$6.01e-03 & 99.10\%$\pm$0.29\% \\
& LySep-FNN & 6.75e-03$\pm$3.00e-03 & 99.83\%$\pm$0.15\% \\
\midrule
\multirow{2}{*}{$(3,16)$}
& CE-FNN    & 8.03e-02$\pm$4.51e-03 & 98.77\%$\pm$0.31\% \\
& LySep-FNN & 3.63e-03$\pm$1.11e-03 & 99.96\%$\pm$0.05\% \\
\midrule
\multirow{2}{*}{$(3,20)$}
& CE-FNN    & 7.62e-02$\pm$5.68e-03 & 99.03\%$\pm$0.30\% \\
& LySep-FNN & 3.46e-03$\pm$8.56e-04 & 99.99\%$\pm$0.02\% \\
\midrule
\multirow{2}{*}{$(10,10)$}
& CE-FNN    & 2.55e-01$\pm$3.03e-01 & 84.09\%$\pm$20.52\% \\
& LySep-FNN & 5.56e-03$\pm$1.90e-03 & 99.86\%$\pm$0.13\% \\
\midrule
\multirow{2}{*}{$(20,10)$}
& CE-FNN    & 6.92e-01$\pm$1.52e-03 & 52.27\%$\pm$1.56\% \\
& LySep-FNN & 3.39e-02$\pm$1.16e-02 & 99.40\%$\pm$0.33\% \\
\bottomrule
\end{tabular}
\caption{\em Final training cross-entropy loss and training accuracy for Example~1. Results are reported as mean $\pm$ standard deviation over $10$ seeds.}
\label{Tab_Case1_train_meanstd}
\end{table}
\begin{figure}[!ht]
\centering
\subfloat[$L=3,M=10$]{\includegraphics[scale=0.17]{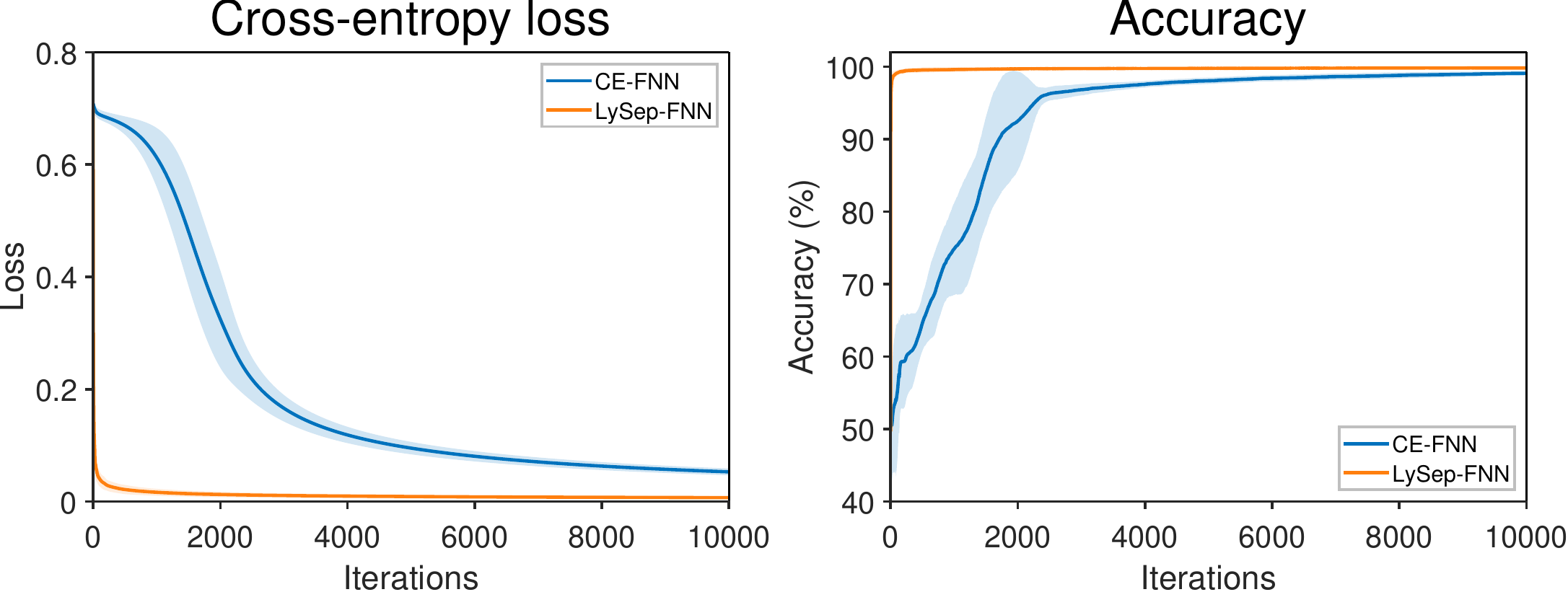}}
\subfloat[$L=20,M=10$]{\includegraphics[scale=0.17]{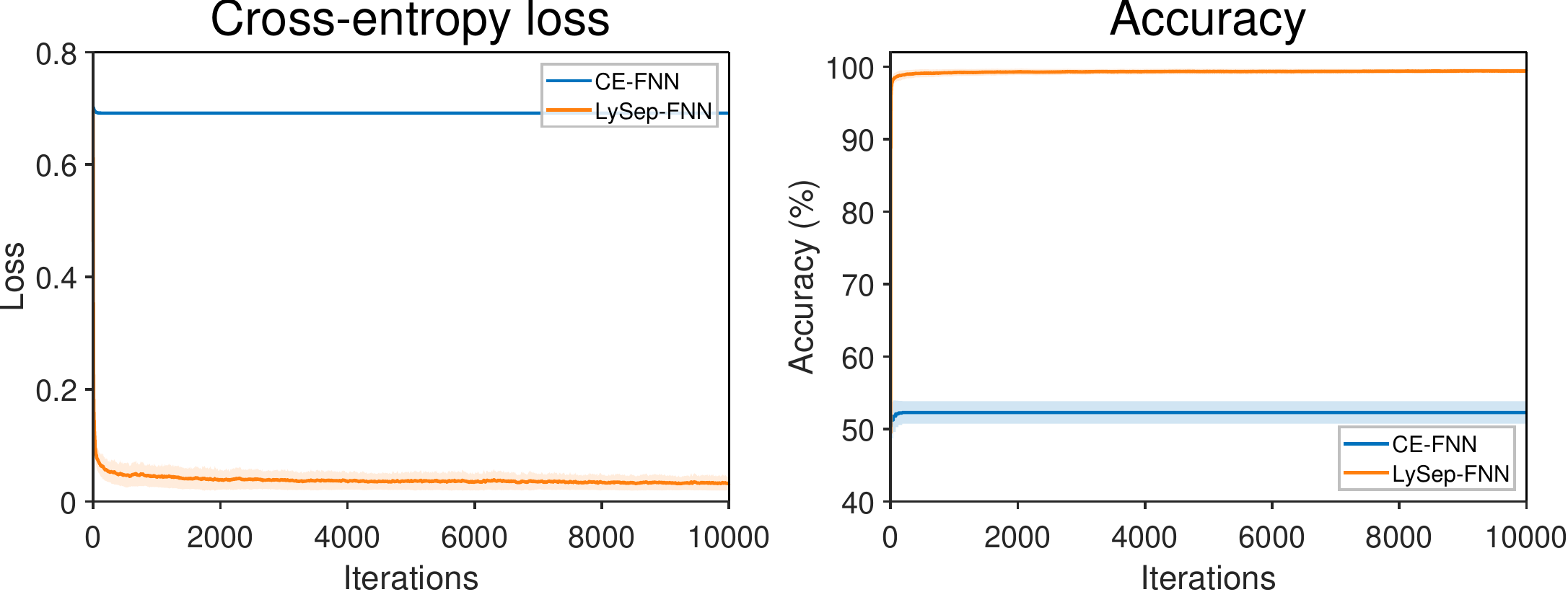}}
\caption{\em Mean training cross-entropy loss and training accuracy versus iterations for Example~1. Shaded bands indicate one standard deviation over $10$ seeds.}
\label{Fig_2}
\end{figure}
\begin{figure}[!ht]
\centering
\subfloat[$L=3,M=10$]{
\includegraphics[scale=0.13]{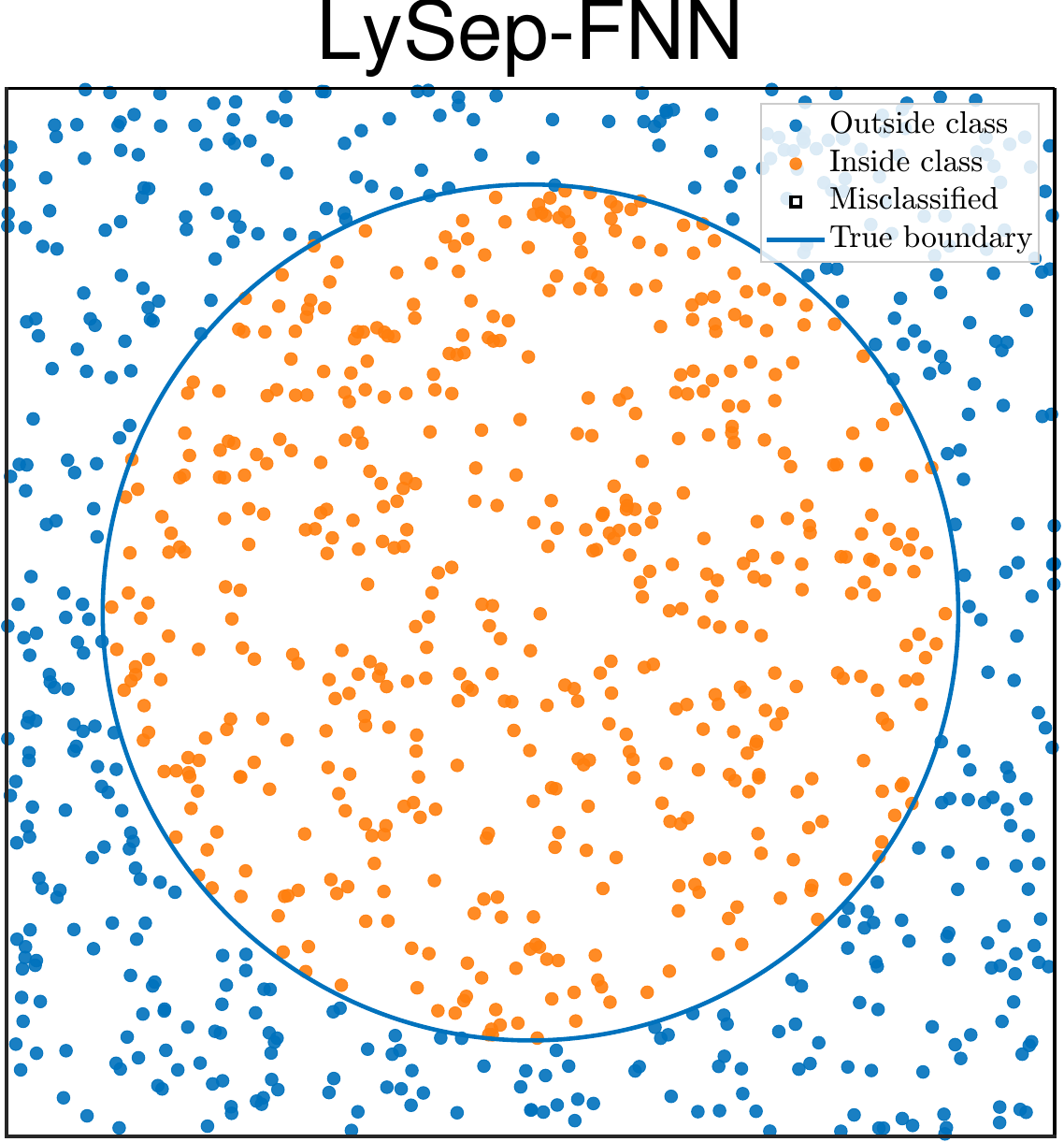}\quad
\includegraphics[scale=0.13]{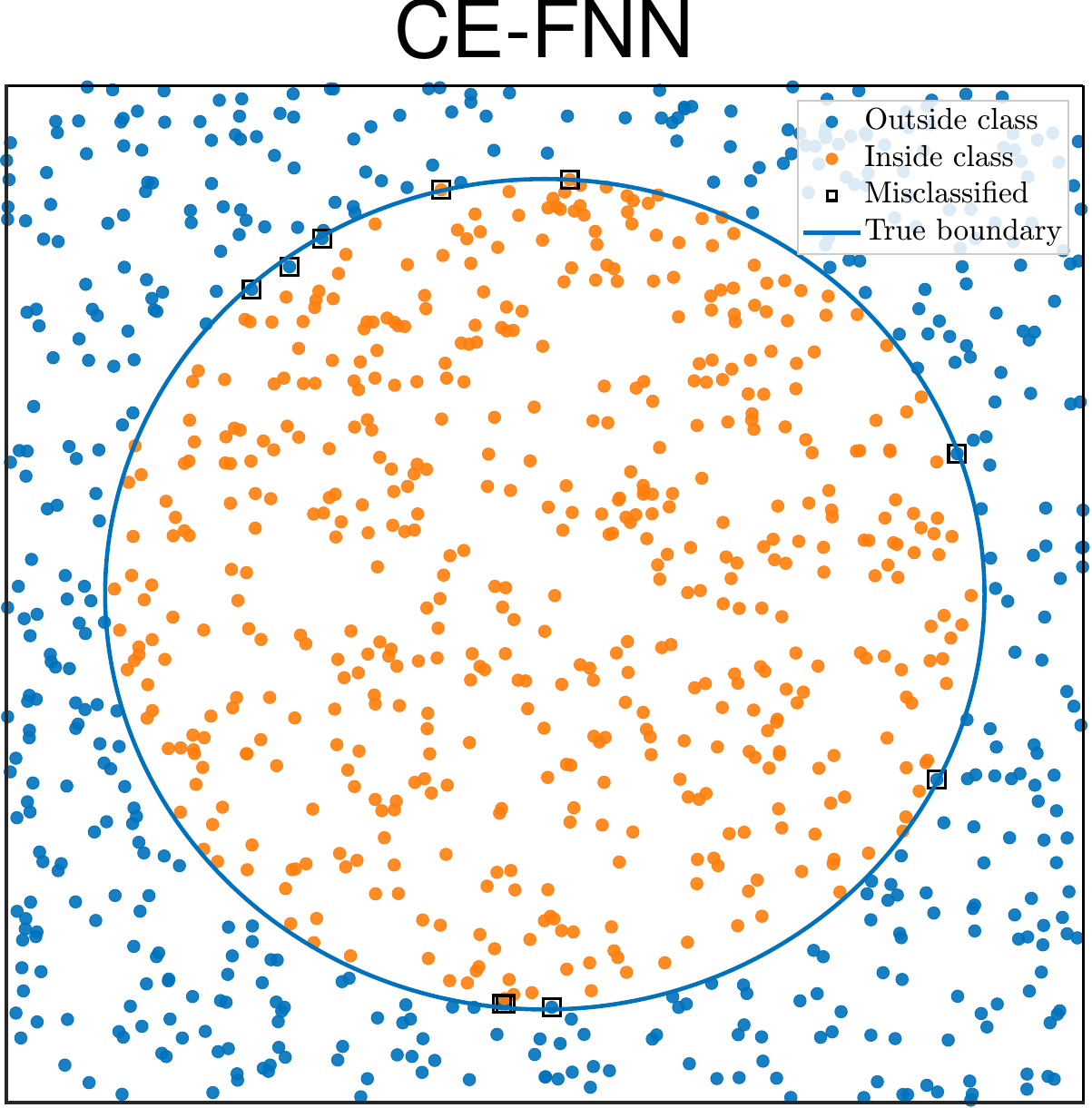}}\quad
\subfloat[$L=20,M=10$]{
\includegraphics[scale=0.13]{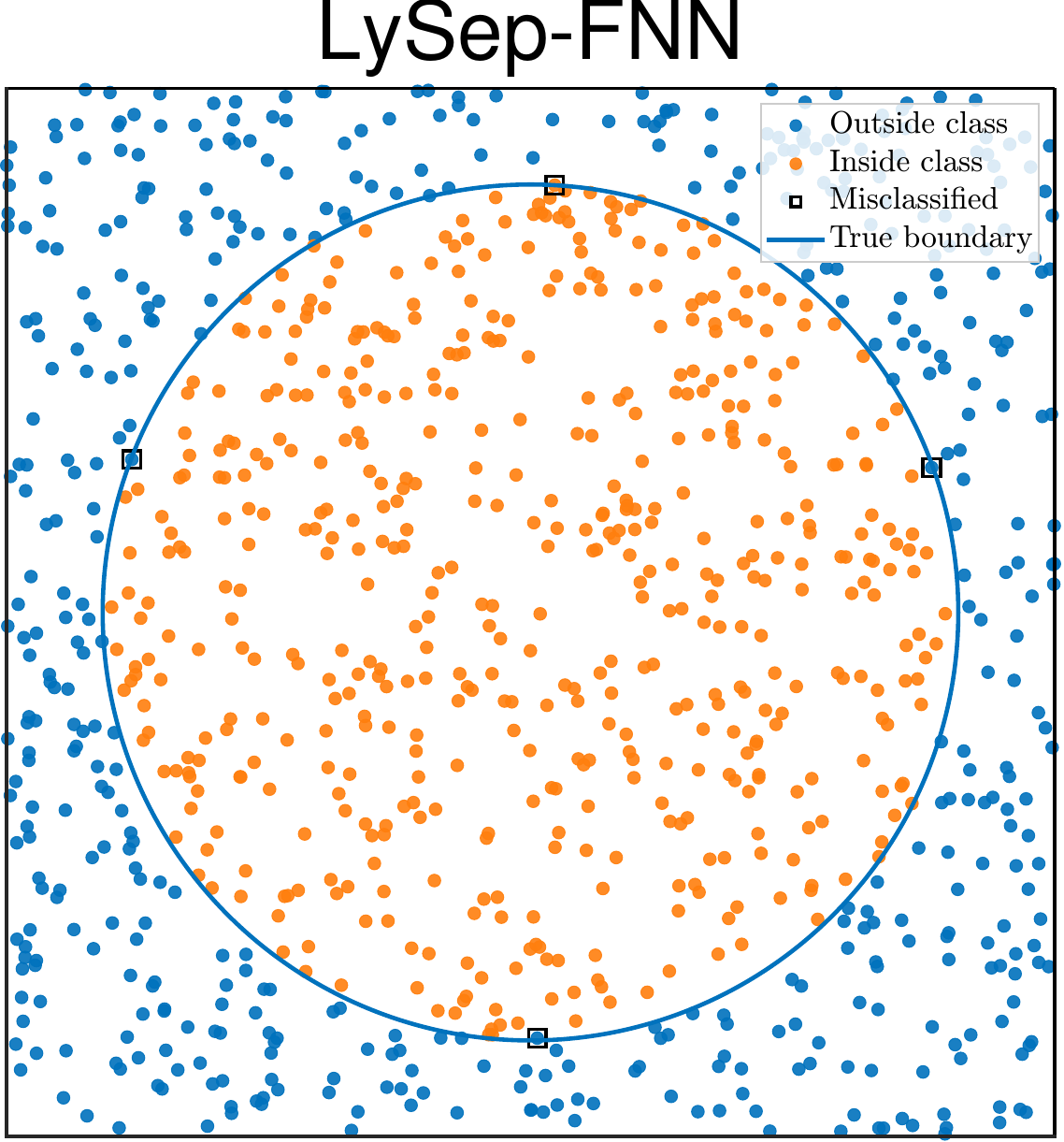}\quad
\includegraphics[scale=0.13]{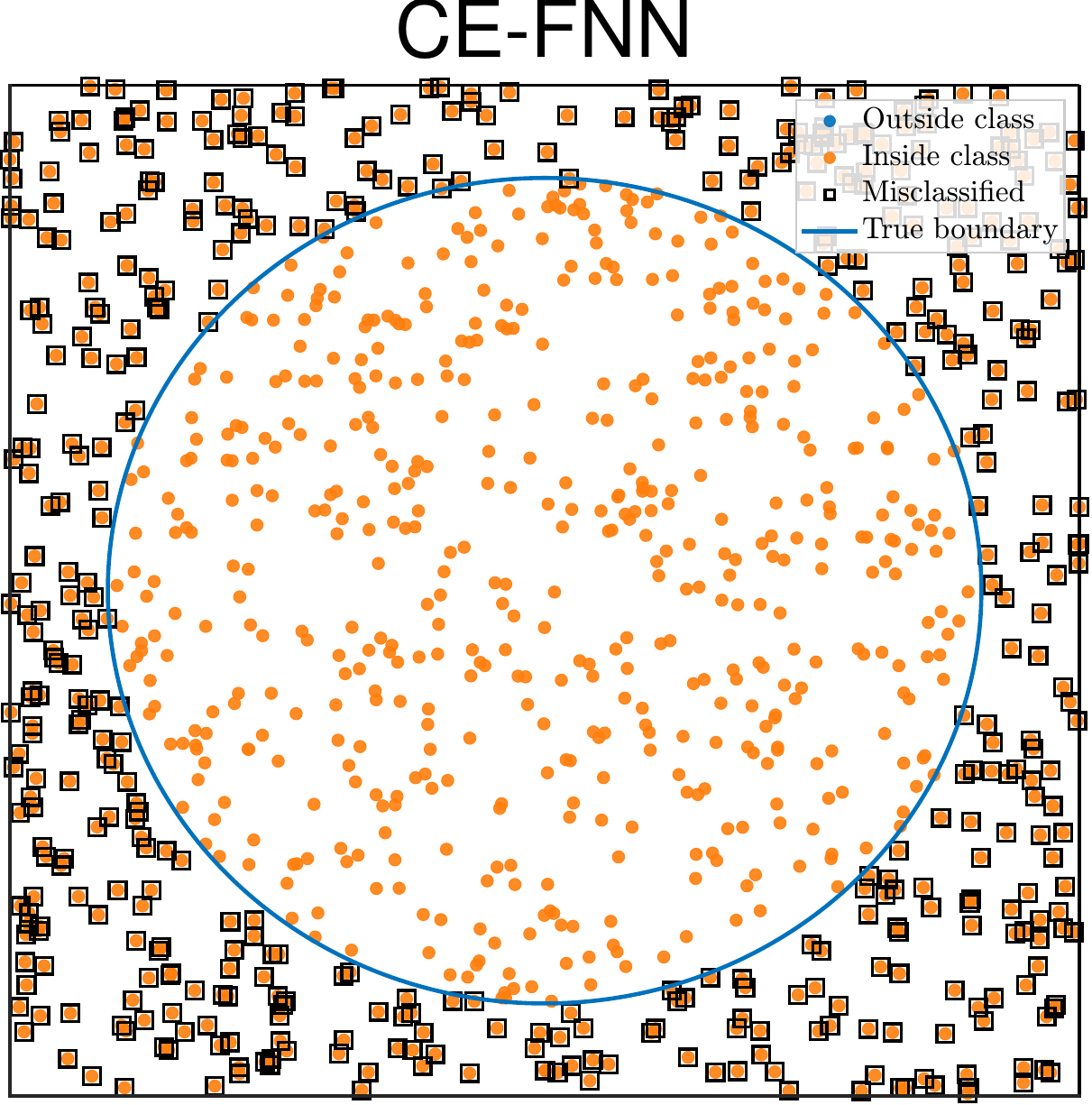}}
\caption{\em Classification results on the test set for Example 1. Black squares indicate misclassified samples. (Results are from the best seed.)}
\label{Fig_3}
\end{figure}

Table~\ref{Tab_Case1_train_meanstd} and Figures~\ref{Fig_2}-\ref{Fig_3} show that LySep-FNN consistently outperforms CE-FNN across all tested depth-width combinations in both training loss and training accuracy. For shallow networks $(L,M)=(3,10)$, $(3,16)$, and $(3,20)$, both methods achieve high training accuracy, but LySep-FNN achieves a lower final loss, converges faster, reaches high training accuracy earlier, and exhibits a smaller standard deviation, indicating better stability compared to random initialization. As the network depth increases, the advantage of LySep-FNN becomes more pronounced. In particular, for $(L,M)=(10,10)$ and $(20,10)$,  CE-FNN exhibits significant optimization difficulties, with final training accuracies decreasing to $84.09\%$ and $52.27\%$, respectively, while LySep-FNN maintains an accuracy above $99\%$ in both cases. Meanwhile, the final training loss of LySep-FNN is approximately $\mathcal{O}(10^{-3})$ and $\mathcal{O}(10^{-2})$, while CE-FNN remains at the level of $\mathcal{O}(10^{-1})$. For the deep network $(L,M)=(20,10)$, Figure~\ref{Fig_2} further shows that LySep-FNN still converges rapidly and stably, while CE-FNN exhibits optimization stagnation, with its loss values remaining nearly at the initial level and training accuracy approaching random levels. This is further confirmed by Figure~\ref{Fig_3}, where LySep-FNN closely matches the true label distribution, whereas CE-FNN produces many misclassified samples. In summary, these results demonstrate that LySep-FNN achieves better optimization performance, stability, and robustness, especially for deep architectures.

\subsubsection{Interface tracking}
In this example, we consider the following Allen-Cahn equation on $\Omega=[-0.5,0.5]^2$ equipped
with periodic boundary condition:
\begin{equation*}
\begin{cases}
u_t=\Delta u+\frac{1}{\varepsilon^2}(u-u^3),\qquad &(x,y)\in\Omega,\ t\in(0,1],\\
u(x,y,0)=u_0,\qquad &(x,y)\in\Omega.
\end{cases}
\end{equation*}
We set $\varepsilon=0.02$ and choose the initial value
\begin{equation*}
u_0(x,y)=
-\tanh\left(\frac{\sqrt{x^2+(y-0.2)^2}-0.19}{\sqrt{2}\varepsilon}\right)
-\tanh\left(\frac{\sqrt{x^2+(y+0.2)^2}-0.19}{\sqrt{2}\varepsilon}\right)+1.
\end{equation*}
We consider the solution interface $u=0$ at $t=0$, $0.001$, $0.004$ and $0.015$. At each time slice, we solve the interface-in/out classification problem using CE-FNN and LySep-FNN. The sample points are selected uniformly in $\Omega$ and labeled according to $\bal_n=\mathbf{1}_{\{u(x_n,y_n,t)\ge 0\}}$. The training and testing sets contain $3000$ and $1000$ samples, respectively. We apply LySep-FNN and CE-FNN to this PDE-induced classification problem and train both models for $10^4$ iterations. To examine the effect of network architecture, we consider the depth-width combinations $(L,M)=(3,30)$, $(10,20)$, and $(20,20)$, and use the $\tanh$ activation function. As in the previous example, we record the cross-entropy loss and training accuracy throughout the training process to evaluate the optimization performance of the two models.

\begin{table}[!ht]
\fontsize{7}{10.5}\selectfont
\setlength{\tabcolsep}{6pt}
\centering
\begin{tabular}{cclcc}
\toprule
$(L,M)$ &$t$& Models & Cross-entropy loss & Accuracy \\
\midrule
\multirow{8}{*}{$(3,30)$}&\multirow{2}{*}{$0$}
& CE-FNN    & 1.50e-01$\pm$1.05e-02 & 94.52\%$\pm$0.47\%\\
&& LySep-FNN & 1.67e-02$\pm$1.96e-03 & 99.43\%$\pm$0.10\% \\\cmidrule(lr){2-5}
&\multirow{2}{*}{$0.001$}
& CE-FNN    &1.25e-01$\pm$9.05e-03  &95.27\%$\pm$0.57\%  \\
&& LySep-FNN & 1.41e-02$\pm$2.73e-03 & 99.48\%$\pm$0.15\% \\\cmidrule(lr){2-5}
&\multirow{2}{*}{$0.004$}
& CE-FNN    &8.22e-02$\pm$6.91e-03  & 96.89\%$\pm$0.33\% \\
&& LySep-FNN & 8.33e-03$\pm$2.68e-03 & 99.69\%$\pm$0.06\% \\\cmidrule(lr){2-5}
&\multirow{2}{*}{$0.015$}
& CE-FNN    & 4.75e-02$\pm$6.11e-03 & 99.20\%$\pm$0.30\% \\
&& LySep-FNN & 2.06e-03$\pm$4.36e-04 & 99.98\%$\pm$0.03\% \\
\midrule
\multirow{8}{*}{$(10,20)$}&\multirow{2}{*}{$0$}
& CE-FNN    & 4.30e-01$\pm$1.81e-01 & 81.92\%$\pm$8.43\% \\
&& LySep-FNN & 4.12e-02$\pm$3.47e-03 & 98.43\%$\pm$0.25\% \\\cmidrule(lr){2-5}
&\multirow{2}{*}{$0.001$}
& CE-FNN    & 3.26e-01$\pm$2.25e-01 & 86.35\%$\pm$9.95\% \\
&& LySep-FNN & 3.82e-02$\pm$3.99e-03 & 98.41\%$\pm$0.27\% \\\cmidrule(lr){2-5}
&\multirow{2}{*}{$0.004$}
& CE-FNN    & 4.21e-01$\pm$1.84e-01 & 82.70\%$\pm$7.68\% \\
&& LySep-FNN & 2.01e-02$\pm$2.14e-03 & 99.14\%$\pm$0.09\% \\\cmidrule(lr){2-5}
&\multirow{2}{*}{$0.015$}
& CE-FNN    & 2.06e-01$\pm$1.88e-01 & 93.02\%$\pm$6.78\% \\
&& LySep-FNN & 6.63e-03$\pm$9.33e-04 & 99.75\%$\pm$0.05\% \\
\midrule
\multirow{8}{*}{$(20,20)$}&\multirow{2}{*}{$0$}
& CE-FNN    & 5.43e-01$\pm$1.14e-02 & 76.70\%$\pm$0.95\% \\
&& LySep-FNN & 1.19e-01$\pm$2.64e-02 & 96.49\%$\pm$1.41\% \\\cmidrule(lr){2-5}
&\multirow{2}{*}{$0.001$}
& CE-FNN    & 5.35e-01$\pm$1.18e-02 & 77.30\%$\pm$0.95\% \\
&& LySep-FNN & 1.15e-01$\pm$2.64e-02 & 96.11\%$\pm$1.50\% \\\cmidrule(lr){2-5}
&\multirow{2}{*}{$0.004$}
& CE-FNN    & 5.10e-01$\pm$1.15e-02 & 79.29\%$\pm$0.86\% \\
&& LySep-FNN & 9.06e-02$\pm$1.02e-02 & 96.57\%$\pm$0.60\% \\\cmidrule(lr){2-5}
&\multirow{2}{*}{$0.015$}
& CE-FNN    & 3.97e-01$\pm$1.51e-02 & 86.44\%$\pm$0.81\% \\
&& LySep-FNN & 2.61e-02$\pm$3.62e-03 & 99.41\%$\pm$0.14\%\\
\bottomrule
\end{tabular}
\caption{\em Final training cross-entropy loss and training accuracy for Example~2. Results are reported as mean $\pm$ standard deviation over $10$ seeds.}
\label{Tab_Case2_train_meanstd}
\end{table}

\begin{figure}[!ht]
\centering
\subfloat[numerical solutions]{
\includegraphics[scale=0.21]{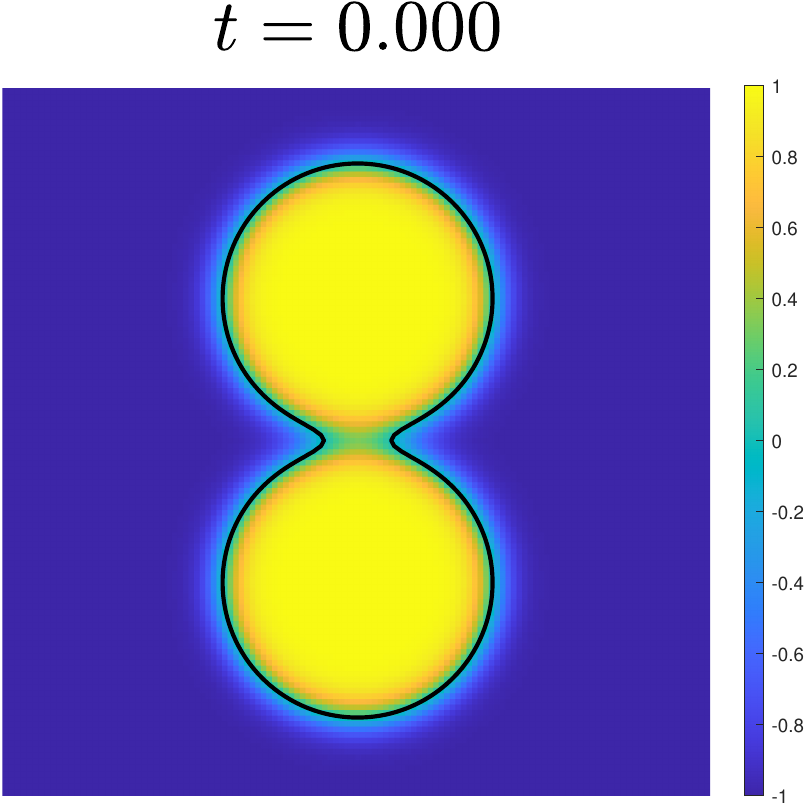}\quad
\includegraphics[scale=0.21]{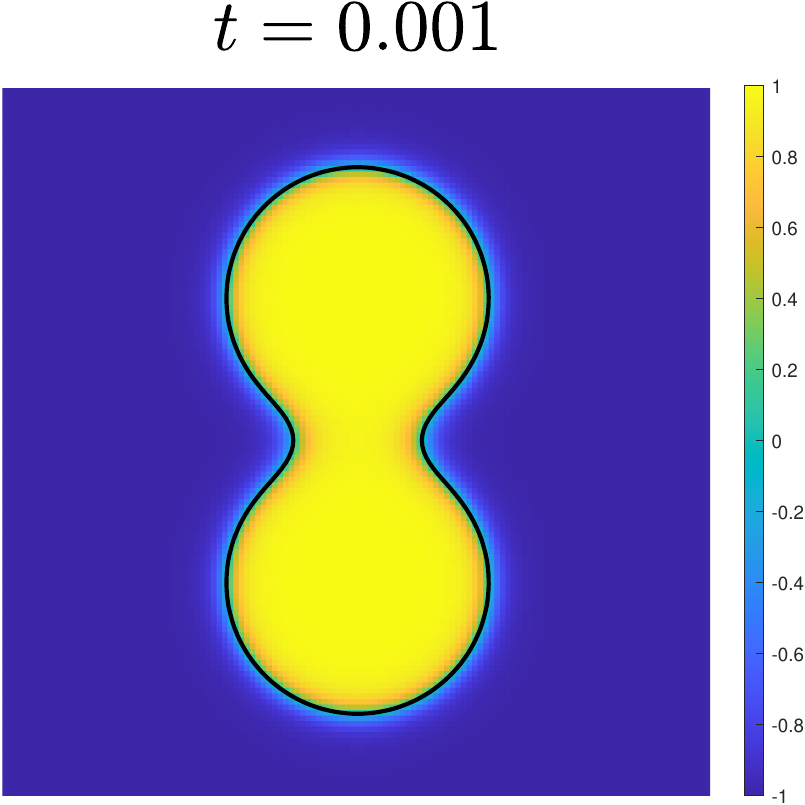}\quad
\includegraphics[scale=0.21]{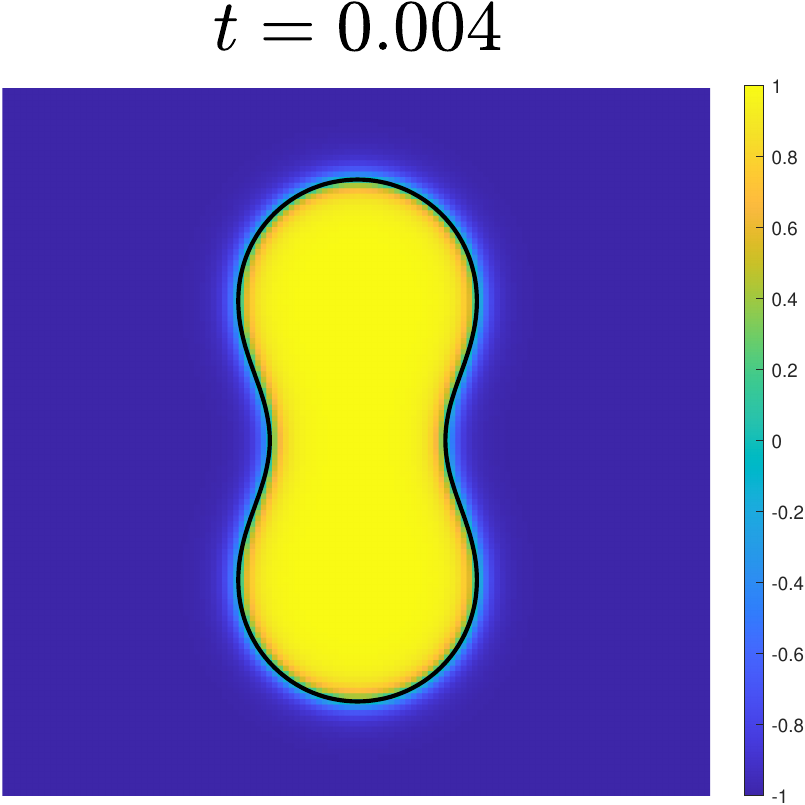}\quad
\includegraphics[scale=0.21]{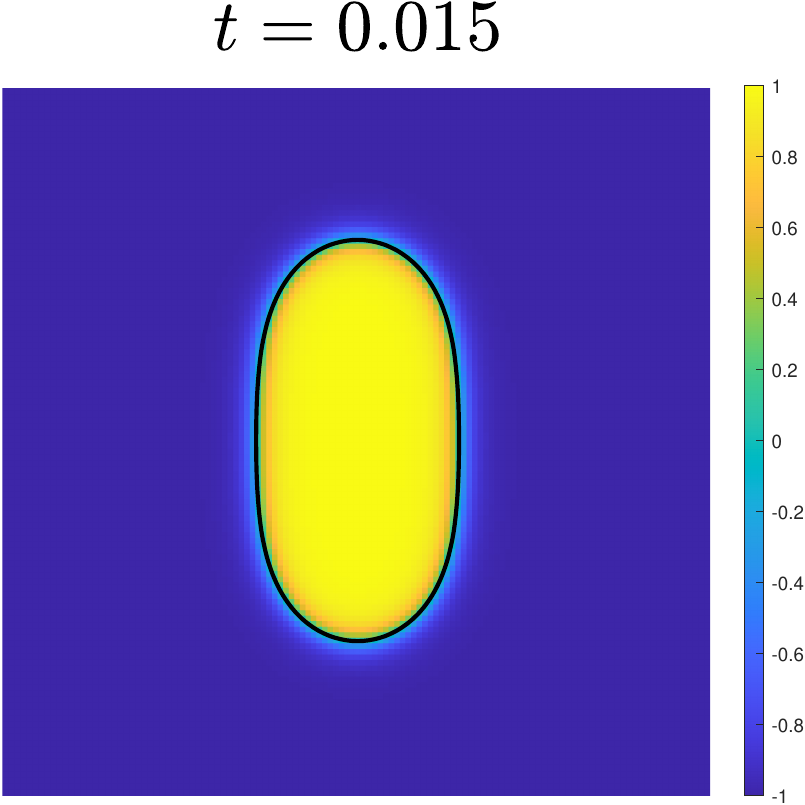}}\\
\subfloat[CE-FNN]{
\includegraphics[scale=0.20]{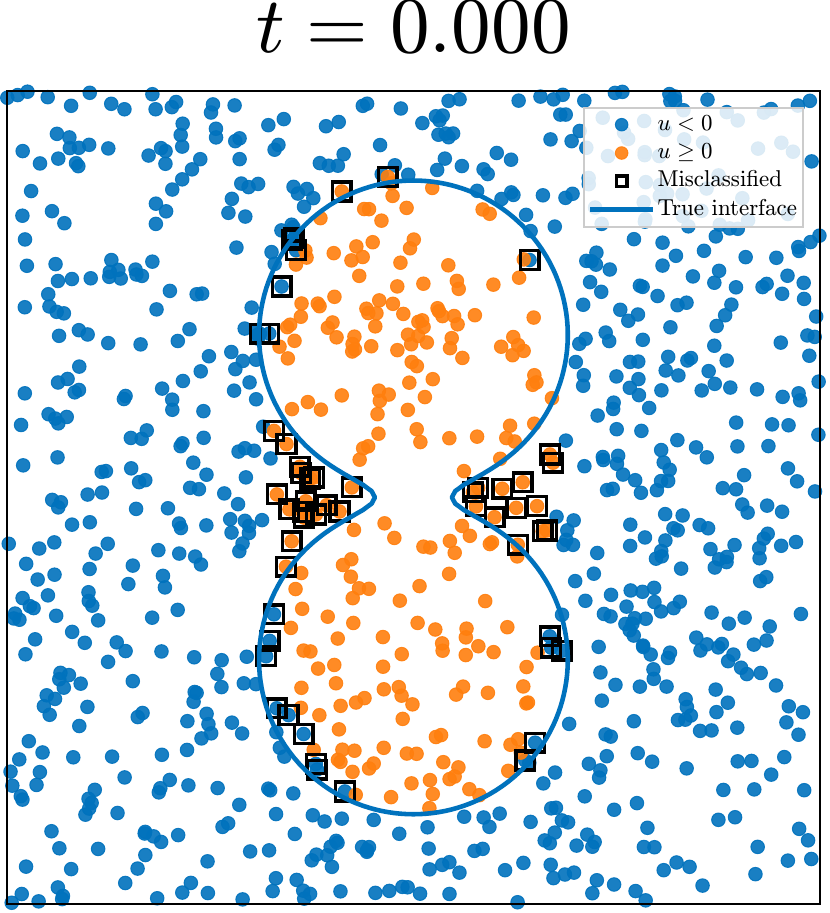}\quad
\includegraphics[scale=0.20]{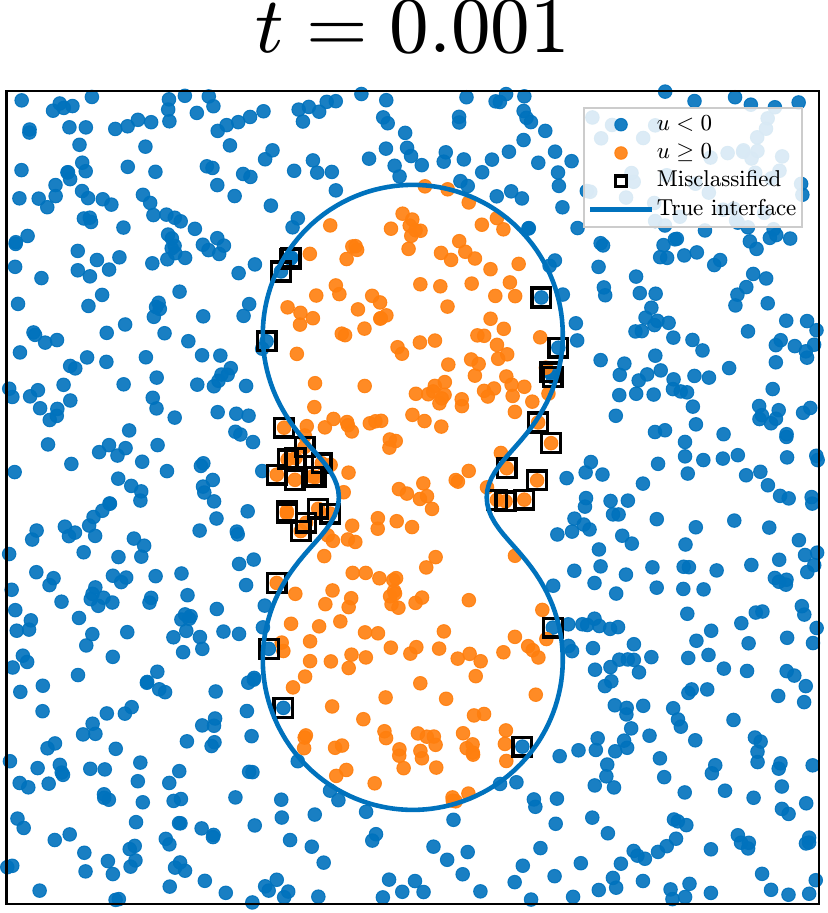}\quad
\includegraphics[scale=0.20]{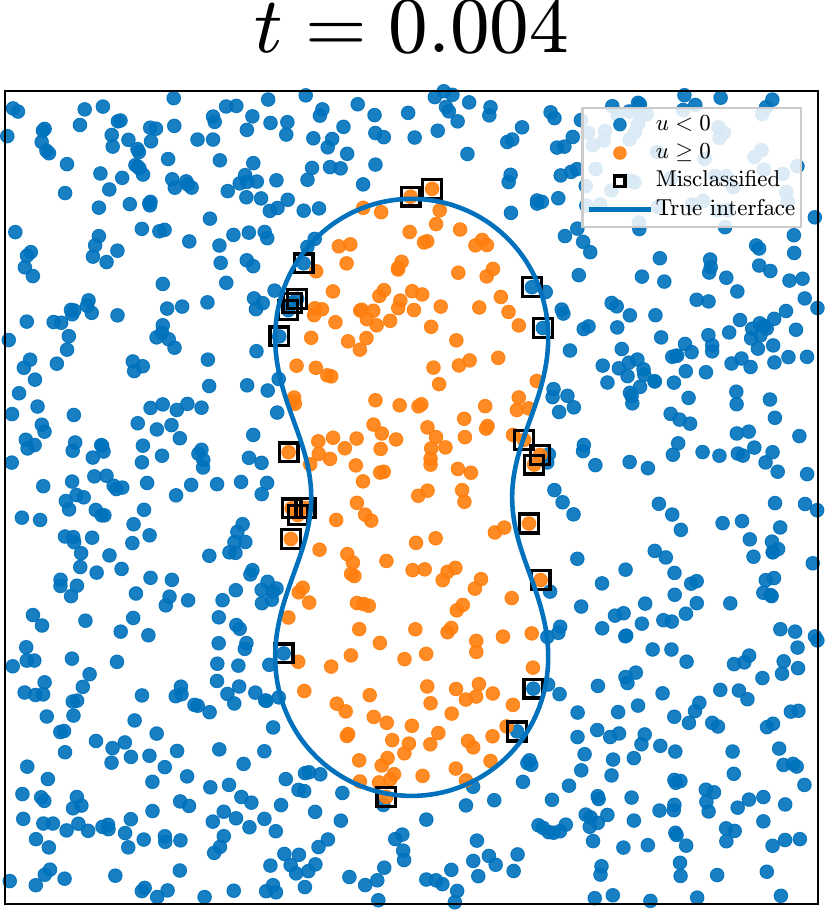}\quad
\includegraphics[scale=0.20]{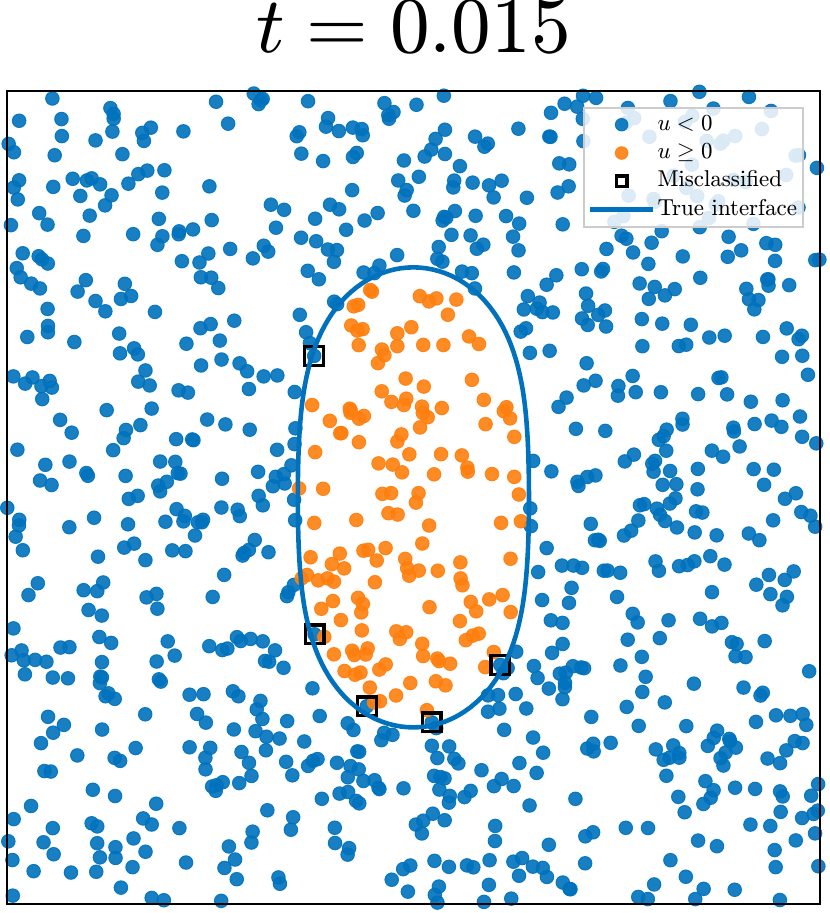}}\\
\subfloat[LySep-FNN]{
\includegraphics[scale=0.20]{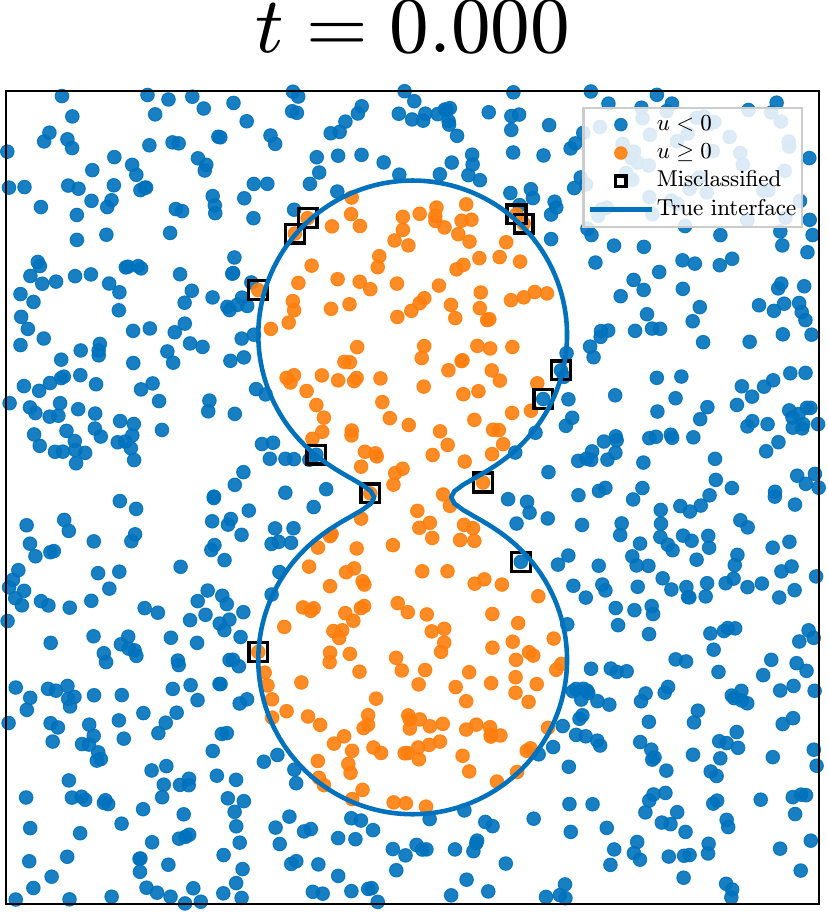}\quad
\includegraphics[scale=0.20]{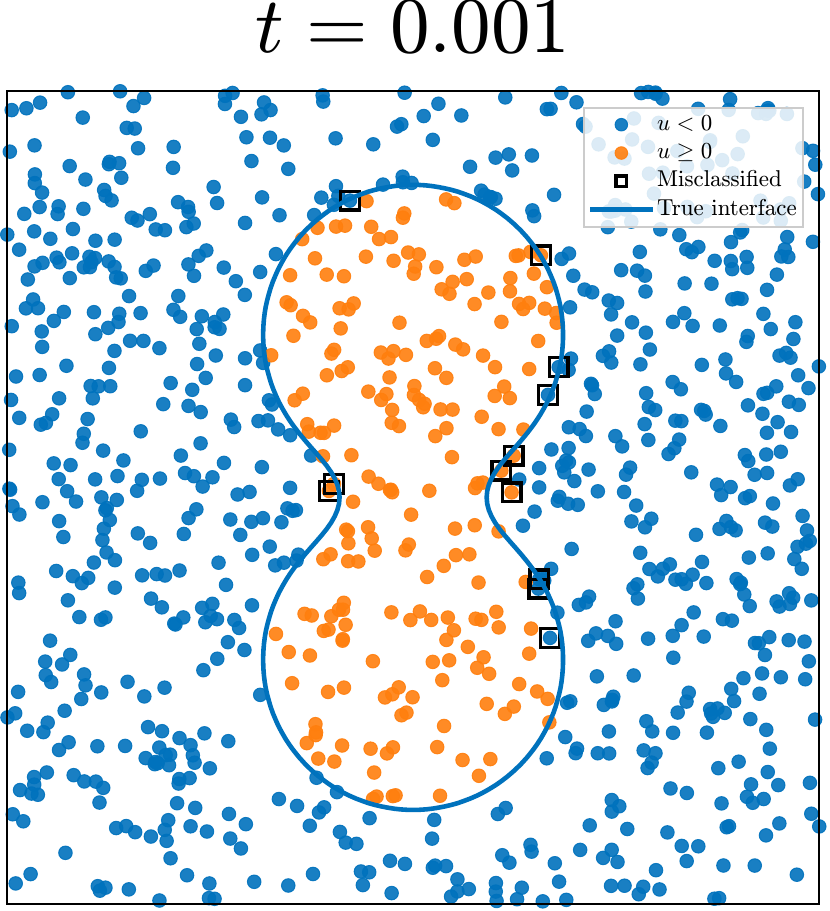}\quad
\includegraphics[scale=0.20]{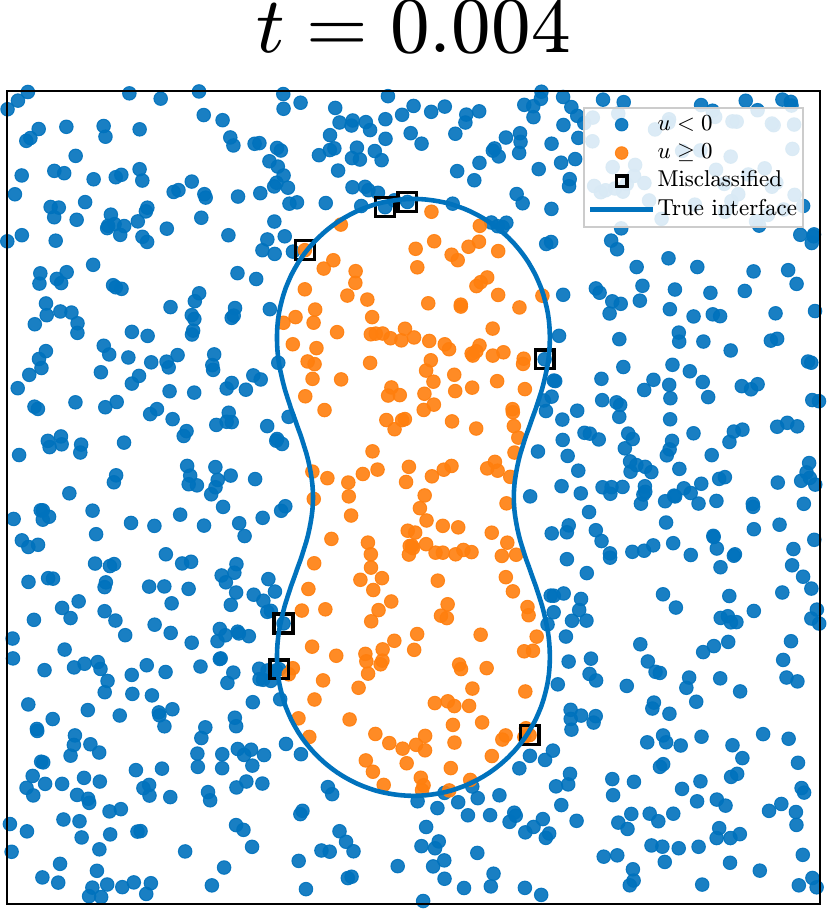}\quad
\includegraphics[scale=0.20]{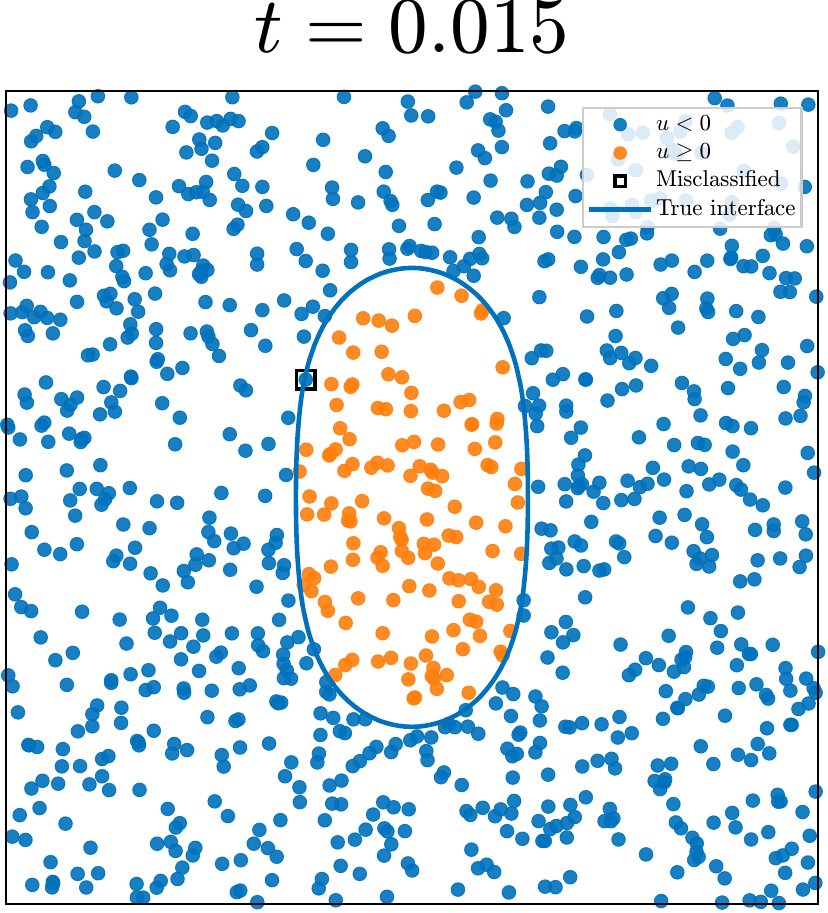}}\\
\caption{\em Numerical solutions and classification results on the test set for Example 2. (Results are from the best seed.)}
\label{Fig_4}
\end{figure}

Table \ref{Tab_Case2_train_meanstd} and Figure \ref{Fig_4} show that the advantage of LySep-FNN is present at all four time slices and becomes more pronounced as the network depth increases. For the shallow architecture $(L,M)=(3,30)$, both methods achieve high training accuracy, but LySep-FNN consistently attains a smaller final loss. For the deeper settings $(L,M)=(10,20)$ and $(20,20)$, the gap is much larger. For example, at $t=0$, the training accuracy improves from $81.92\%$ to $98.43\%$ for $(10,20)$ and from $76.70\%$ to $96.49\%$ for $(20,20)$, while the corresponding training loss is reduced from $4.30\times 10^{-1}$ to $4.12\times 10^{-2}$ and from $5.43\times 10^{-1}$ to $1.19\times 10^{-1}$, respectively. Figure \ref{Fig_4} shows that these quantitative improvements correspond to visibly more accurate interface tracking, especially at early time slices and for deeper networks, where CE-FNN produces many misclassified points near the moving interface. Overall, these results show that the proposed model improves both optimization stability and interface reconstruction quality on this PDE-induced classification task.

\subsection{Multi-class digit classification with CNNs}
In the third example, we consider the MNIST handwritten-digit dataset, a standard benchmark for multi-class image classification. We compare the conventional CE-CNN model with the proposed LySep-CNN model. The training set contains $60000$ images, and both methods are trained for $10^4$ iterations. During training, we use mini-batch gradient descent with batch size $1500$, and the training subset is refreshed every $200$ iterations.

To investigate the effects of network depth and channel size on the optimization behavior, we fix the number of fully connected layers at $L_2=2$ and consider three choices for the number of convolutional layers, namely $L_1=2,3,4$, together with three channel settings, $C=1,2,4$, where $C$ denotes the number of output channels in each convolutional layer. The convolution kernel size is set to $3$, the stride to $1$, and the pooling size to $2$, while the width of the hidden fully connected layer is fixed at $50$. To ensure a fair comparison across different CNN architectures, we keep the feature dimensions at the input to the fully connected layers consistent. This allows us to study the effects of varying the depth and channel size of the convolutional layers without interference from the fully connected layers. For each depth-channel configuration, we record the cross-entropy loss and training classification accuracy throughout the iterative process to compare the optimization performance of the two methods on this multiclass classification task.

\begin{table}[!ht]
\fontsize{7}{10.5}\selectfont
\setlength{\tabcolsep}{6pt}
\centering
\begin{tabular}{cclcc}
\toprule
$L_1$ &$C$& Models & Cross-entropy loss & Accuracy \\
\midrule
\multirow{8}{*}{$2$}&\multirow{2}{*}{$1$}
& CE-CNN     &3.24e-01$\pm$2.62e-02 & 90.52\%$\pm$0.92\%\\
&& LySep-CNN &2.15e-01$\pm$4.94e-02 & 94.39\%$\pm$1.76\%\\\cmidrule(lr){2-5}
&\multirow{2}{*}{$2$}
& CE-CNN     &2.60e-01$\pm$2.89e-02 & 92.39\%$\pm$1.04\% \\
&& LySep-CNN &8.67e-02$\pm$2.12e-02 &99.32\%$\pm$0.37\%\\\cmidrule(lr){2-5}
&\multirow{2}{*}{$4$}
& CE-CNN     &1.96e-01$\pm$2.13e-02 & 94.42\%$\pm$0.94\% \\
&& LySep-CNN &2.92e-02$\pm$4.87e-03 &99.80\%$\pm$0.09\%\\
\midrule
\multirow{8}{*}{$3$}&\multirow{2}{*}{$1$}
& CE-CNN     &3.04e-01$\pm$2.68e-02 & 90.77\%$\pm$0.82\% \\
&& LySep-CNN &2.02e-01$\pm$4.44e-02 &95.28\%$\pm$1.80\%\\\cmidrule(lr){2-5}
&\multirow{2}{*}{$2$}
& CE-CNN     &2.18e-01$\pm$2.65e-02 & 93.51\%$\pm$1.03\%\\
&& LySep-CNN &8.81e-02$\pm$3.51e-02 & 99.15\%$\pm$0.95\%\\\cmidrule(lr){2-5}
&\multirow{2}{*}{$4$}
& CE-CNN     &1.64e-01$\pm$1.69e-02 & 95.37\%$\pm$0.46\%\\
&& LySep-CNN &2.80e-02$\pm$9.59e-03 & 99.83\%$\pm$0.08\%\\
\midrule
\multirow{8}{*}{$4$}&\multirow{2}{*}{$1$}
& CE-CNN     &2.79e-01$\pm$3.30e-02 & 91.70\%$\pm$0.95\% \\
&& LySep-CNN &2.32e-01$\pm$7.19e-02 &
94.08\%$\pm$2.62\%\\\cmidrule(lr){2-5}
&\multirow{2}{*}{$2$}
& CE-CNN     &2.09e-01$\pm$2.84e-02 & 94.08\%$\pm$1.12\% \\
&& LySep-CNN &9.82e-02$\pm$3.47e-02 & 98.91\%$\pm$0.87\%\\\cmidrule(lr){2-5}
&\multirow{2}{*}{$4$}
& CE-CNN     &1.43e-01$\pm$1.71e-02 & 95.97\%$\pm$0.65\%\\
&& LySep-CNN &5.09e-02$\pm$2.82e-02 & 99.68\%$\pm$0.14\%\\
\bottomrule
\end{tabular}
\caption{\em Final training cross-entropy loss and training accuracy for Example~3. Results are reported as mean $\pm$ standard deviation over $10$ seeds.}
\label{Tab_Case3_train_meanstd}
\end{table}

Table \ref{Tab_Case3_train_meanstd} shows that LySep-CNN achieves lower final training loss and higher training accuracy than CE-CNN for all tested depth-channel configurations on the MNIST dataset. For $C=1$, the improvement is already visible for each choice of $L_1$, while for $C=2$ and $C=4$ the gap becomes more significant. In particular, when $(L_1,C)=(2,4)$, the final training loss decreases from $1.96\times 10^{-1}$ to $2.92\times 10^{-2}$ and the training accuracy increases from $94.42\%$ to $99.80\%$. When $(L_1,C)=(4,4)$, the corresponding values change from $1.43\times 10^{-1}$ to $5.09\times 10^{-2}$ and from $95.97\%$ to $99.68\%$, respectively. These results indicate that the proposed layer separation framework becomes increasingly effective as the CNN architecture deepens or widens.

\section{Conclusion}\label{sec:5}
This paper proposes a layer separation framework for deep learning optimization with the softmax cross-entropy loss. To address the high non-convexity inherent in deep neural networks, we introduce auxiliary variables to represent the outputs of hidden layers and construct layer separation models for fully connected and convolutional neural networks. The models retain only couplings between neighboring layers, thereby decomposing the original deeply nested problem into a sequence of simpler subproblems. We prove that the original cross-entropy loss is bounded above by the proposed layer separation loss up to a constant. We also develop alternating minimization algorithms for these models, in which the variables are partly updated in closed form. We further establish the loss-decreasing property of the proposed algorithms. Numerical experiments demonstrate that the proposed models exhibit more favorable optimization behavior than the standard cross-entropy models, especially for deeper neural networks. These results indicate that the proposed framework provides a theoretically justified and practically effective approach for cross-entropy training in deep learning.

Future work can extend this framework to other network architectures (e.g., ResNet and recurrent networks), establish stronger convergence results, and apply the framework to a broader class of scientific machine learning problems to further assess its ability to handle strongly coupled optimization problems.

%\section*{Acknowledgments}
%We would like to acknowledge the assistance of volunteers in putting together this example manuscript and supplement.

\bibliographystyle{siamplain}
\bibliography{references}

\end{document}